\documentclass[sigconf]{acmart}

\setcitestyle{numbers,sort&compress}

\usepackage{mathtools}
\usepackage{algorithm}
\usepackage{algorithmic}
\usepackage{pgfplots}
\usepackage[absolute,overlay]{textpos}

\usepackage{graphicx}

\usepackage{url}

\usepackage{cleveref}
\crefname{figure}{Fig.}{Figs.}
\crefname{table}{Tab.}{Tabs.}
\crefname{equation}{Eq.}{Eqs.}
\usepackage{subcaption}

\usepackage{bm}
\usepackage{amsmath}
\usepackage{makecell}
\usepackage{array}
\usepackage{multirow}
\usepackage{caption} 
\usepackage{amsthm}
\usepackage{pgfplots}
\usepackage{natbib}
\usepackage{mathtools}
\usepackage{graphicx}
\usepackage{threeparttable}
\usepackage{enumitem}
\setlist[itemize]{noitemsep, topsep=0pt}

\usepackage{xcolor}         
\usepackage{hyperref}
\hypersetup{
    colorlinks,
    linkcolor={red!50!black},
    citecolor={blue!50!black},
    urlcolor={blue!80!black}
}

\newtheorem{theorem}{\textbf{Theorem}}

\newtheorem{lemma}{\textbf{Lemma}}

\pgfplotsset{compat=1.18}

\newcommand{\R}{\mathbb{R}}

\newcommand{\E}{\mathbb{E}}

\newcommand{\I}{\mathbb{I}}

\newcommand{\bb}{\boldsymbol{b}}

\newcommand{\bch}{\boldsymbol{\chi}}

\newcommand{\bs}{\boldsymbol{s}}

\newcommand{\bg}{\boldsymbol{g}}

\newcommand{\bX}{\boldsymbol{X}}
\newcommand{\bY}{\boldsymbol{Y}}
\newcommand{\bx}{\boldsymbol{x}}
\newcommand{\by}{\boldsymbol{y}}

\newcommand{\cA}{\mathcal{A}}
\newcommand{\cB}{\mathcal{B}}

\newcommand{\cD}{\mathcal{D}}
\newcommand{\cE}{\mathcal{E}}

\newcommand{\cG}{\mathcal{G}}

\newcommand{\cN}{\mathcal{N}}
\newcommand{\cU}{\mathcal{U}}
\newcommand{\cV}{\mathcal{V}}

\newcommand{\cQ}{{\mathcal{Q}}}

\newcommand{\bmu}{\boldsymbol{\mu}}

\newcommand{\bbR}{\mathbb{R}}

\DeclareMathOperator*{\argmax}{arg\,max}

\newcommand{\tDelta}{\tilde{\Delta}}

\setcopyright{acmlicensed}
\copyrightyear{2018}
\acmYear{2018}
\acmDOI{XXXXXXX.XXXXXXX}

\acmConference[Conference acronym 'XX]{Make sure to enter the correct
  conference title from your rights confirmation emai}{June 03--05,
  2018}{Woodstock, NY}

\copyrightyear{2025}
\acmYear{2025}
\setcopyright{cc}
\setcctype{by}
\acmConference[KDD '25]{Proceedings of the 31st ACM SIGKDD Conference on Knowledge Discovery and Data Mining V.2}{August 3--7, 2025}{Toronto, ON, Canada}
\acmBooktitle{Proceedings of the 31st ACM SIGKDD Conference on Knowledge Discovery and Data Mining V.2 (KDD '25), August 3--7, 2025, Toronto, ON, Canada}
\acmDOI{10.1145/3711896.3736824}
\acmISBN{979-8-4007-1454-2/2025/08}

\begin{document}
\title{A Unified Online-Offline Framework for Co-Branding Campaign Recommendations}

\author{Xiangxiang Dai}
\orcid{https://orcid.org/0000-0003-0179-196X}
\affiliation{%
  \institution{The Chinese University of Hong Kong, Hong Kong SAR, China}
  \country{}
  }
\email{xxdai23@cse.cuhk.edu.hk}

\author{Xiaowei Sun}
\orcid{https://orcid.org/0009-0003-2084-0656}
\affiliation{%
   \institution{Fudan University,}
  \country{Shanghai, China}
}
\email{xwsun24@m.fudan.edu.cn}

\author{Jinhang Zuo}
\orcid{https://orcid.org/0000-0002-9557-3551}
\affiliation{%
  \institution{City University of Hong Kong,}
  \country{Hong Kong SAR, China}
}
\email{jinhang.zuo@cityu.edu.hk}

\author{Xutong Liu}
\orcid{https://orcid.org/0000-0002-8628-5873}
\authornote{Xutong Liu is the corresponding author.}
\affiliation{%
  \institution{Carnegie Mellon University,}
  \country{PA, USA}}
  \email{xutongl@andrew.cmu.edu}

  \author{John C.S. Lui}
  \orcid{https://orcid.org/0000-0001-7466-0384}
\affiliation{%
  \institution{The Chinese University of Hong Kong, Hong Kong SAR, China}
  \country{}}
  \email{cslui@cse.cuhk.edu.hk}


\begin{abstract}
Co-branding has become a vital strategy for businesses aiming to expand market reach within recommendation systems. However, identifying effective cross-industry partnerships remains challenging due to resource imbalances, uncertain brand willingness, and ever-changing market conditions. In this paper, we provide the first systematic study of this problem and propose a unified online-offline framework to enable co-branding recommendations. Our approach begins by constructing a bipartite graph linking ``initiating'' and ``target'' brands to quantify co-branding probabilities and assess market benefits. During the online learning phase, we dynamically update the graph in response to market feedback, while striking a balance between exploring new collaborations for long-term gains and exploiting established partnerships for immediate benefits. To address the high initial co-branding costs, our framework mitigates redundant exploration, thereby enhancing short-term performance while ensuring sustainable strategic growth. In the offline optimization phase, our framework consolidates the interests of multiple sub-brands under the same parent brand to maximize overall returns, avoid excessive investment in single sub-brands, and reduce unnecessary costs associated with over-prioritizing a single sub-brand.  We present a theoretical analysis of our approach, establishing a highly nontrivial sublinear regret bound for online learning in the complex co-branding problem, and enhancing the approximation guarantee for the NP-hard offline budget allocation optimization. Experiments on both synthetic and real-world co-branding datasets demonstrate the practical effectiveness of our framework, with at least 12\% improvement.
\end{abstract}

\ccsdesc[500]{Information systems~Recommender systems}
\ccsdesc[500]{Theory of computation~Online learning algorithms}
\ccsdesc[500]{Theory of computation~Approximation algorithms analysis}

\keywords{Co-branding recommendation, Online learning, Multi-armed bandit, Offline optimization, Approximation guarantee}

\maketitle

\newcommand\kddavailabilityurl{https://doi.org/10.5281/zenodo.15532548}

\ifdefempty{\kddavailabilityurl}{}{
\begingroup\small\noindent\raggedright\textbf{KDD Availability Link:}\\
The source code of this paper has been made publicly available at \url{\kddavailabilityurl}.
\endgroup
}

\section{Introduction}
Co-branding recommendation is a popular strategy that businesses and consumers alike embrace.  By leveraging collaborations between two different brands, this recommendation strategy involves co-creating new products or services to maximize user satisfaction and achieve joint market benefits \cite{gogri2022co}. 
For example, the launch of \textit{CASETiFY}’s collaborative line of electronic accessories in partnership with \textit{Netflix}'s television series ``Squid Game 2'' strongly resonates with fans of the series, showcasing the success of such brand collaborations\footnote{https://www.casetify.com/co-lab/squid-game-2024}. These collaborations not only provide brands with enhanced exposure and increased market gains, but also cater to consumer demand for novelty. According to SocialBeta's data  \cite{SocialBeta} covering the period from March 2022 to February  2023, approximately 15\% of the 4,428 recorded brand marketing cases involved co-branding. Further analysis reveals that this proportion never dropped below 10\% in any given month, with co-branding seeing notable spikes during key promotional periods like summer sales and the year-end holiday season (e.g., Christmas and New Year). In this context, co-branding has now become a trend.  Despite its potential, a key challenge remains: \textit{How can brands identify the most suitable partners from the set of target brands to ensure successful co-branding and maximize market returns?} This question remains systematically unresolved and cannot be overlooked. For instance, the collaboration between \textit{Heytea} and the \textit{Jingdezhen China Ceramic Museum} resulted in reduced profits,  highlighting the importance of selecting the right target brands \cite{turan2022deal, helmig2007explaining}.

However, this is a challenging task. \textit{The co-branding market is fraught with uncertainty.} Firstly,  overemphasizing certain popular co-branding collaborations can lead to a resource imbalance, i.e., ``Matthew Effect'' \cite{Gao2023AlleviatingME}, where smaller, long-tail brands are overlooked despite their potential to deliver greater market benefits. 
The willingness of target brands to participate in collaborations is also uncertain, influenced by multiple factors such as brand positioning and financial commitment from the initiating brand. Additionally, market changes make over-reliance on historical data risky. For instance, the continued use of certain types of co-branding may lead to consumer fatigue or shifting preferences, which may be overlooked by brands overly dependent on past performance. In such cases, collaborations with new brands can disrupt the status quo and introduce a sense of novelty \cite{bobadilla2013recommender,turan2022deal}.

In such a dynamic online recommendation market, static strategies are insufficient for determining appropriate co-branding partnerships. Instead, \textit{a more dynamic approach that continuously adapts to uncertainty is required}. This involves addressing the ``\textit{exploration-exploitation trade-off}'', where brands must balance the known willingness and market benefits of established partnerships while exploring potential collaborations with new, promising target brands that could deliver better market performance.
Recent studies have shown that multi-armed bandit (MAB) methods can effectively address this trade-off challenge in recommendation systems in a lightweight, online manner \cite{heyden2024budgeted,li2025towards,xia2024kdd,dai2024online}. \textit{From a decision-making perspective, co-branding can be framed as a recommendation problem, where the goal is to suggest partnerships that maximize long-term market impact.} However, unlike traditional recommendation tasks (e.g., recommending a movie to a user), the costs associated with co-branding decisions are significantly high \cite{yang2024conversational,dai2024conversational,liu2024combinatorial}. Implementing a co-branding partnership requires substantial upfront efforts, including extensive market research and creating branded materials, all of which incur significant costs. As a result, while long-term brand objectives may benefit from the online adjustment of co-branding strategies to maximize collaboration gains, exploration must be conducted with caution to mitigate risks—especially in the early stages—to prevent brands from losing confidence in the algorithm before its benefits materialize.

In addition to the inherent uncertainty surrounding the willingness or feasibility of co-branding and market gains, brands must also manage budget constraints when planning co-branding initiatives.  In practice, the parent brand that initiates co-branding campaigns often oversees ``\textit{multiple sub-brands}'' (e.g., the \textit{Coca-Cola} company owns sub-brands like \textit{Sprite} and \textit{Fanta}, or \textit{Marvel} has various superhero intellectual properties (IPs), such as Spider-Man and Iron Man). An illustration of this structure is shown in (a) part of \Cref{fig:model}, with further details provided in \Cref{sec:model}. Due to financial limitations, it may not be feasible for every sub-brand to engage in co-branding with target brands. Consequently, when planning and allocating budgets for a co-branding initiative, the parent brand must adopt a holistic view of the entire brand ecosystem, rather than focusing on a single sub-brand. Focusing solely on co-branding efforts for individual sub-brands could reduce overall collaboration benefits, potentially affecting the profitability of the parent brand as a whole.  In such cases, allocating excessive resources to a single sub-brand's co-branding efforts may lead to diminishing returns, whereas a more cost-effective sub-brand might generate higher overall value for the parent brand.
Therefore, \textit{it is crucial to assess and align the interests of all sub-brands within the parent brand's portfolio.}  Optimizing budget allocation for maximum collective benefit, rather than focusing on individual sub-brands, can help maximize overall market outcomes.

To the best of our knowledge, no existing work systematically addresses the complex problem of co-branding campaign recommendations. We present \textit{the first} systematic approach to model and tackle this challenge. Motivated by the aforementioned practical needs,  we propose a unified online-offline framework, beginning with the construction of a co-branding bipartite graph model. This model captures the structural relationships between the initiating parent brand—along with its multiple sub-brands—and target brands, while incorporating key factors such as the willingness or feasibility of co-branding and market gains to accurately represent brand interactions.
Our framework enables two key processes: online graph learning and offline budget optimization, leveraging their synergy to continuously enhance co-branding strategies.

\textbf{Online Phase.} The model dynamically updates itself based on online market feedback rather than remaining static. 
 Extending beyond deterministic arm selection in traditional MAB methods,  our approach explicitly addresses uncertainties in co-branding, including budget-dependent co-branding willingness and market gains that are revealed only upon successful co-branding.  Given the high upfront investment in co-branding, we also emphasize short-term performance alongside long-term
performance by strategically minimizing redundant exploration. Theoretically, through a highly non-trivial analysis, we merge the error terms introduced by these new uncertainties while maintaining a near-optimal regret bound.
 
\textbf{Offline Phase.} We allocate budgets based on the insights learned from online feedback, optimizing future collaboration outcomes from the parent brand's perspective. This approach avoids the pitfalls of overemphasizing individual sub-brands and mitigates unnecessary premium costs. Given the high expenses of co-branding campaigns—e.g., joint marketing, material production, and negotiation efforts—we enhance the approximation ratio by introducing a partial enumeration technique. This enables more flexible allocations for the NP-hard budget allocation problem while ensuring computational traceability compared to previous works.

Our contributions are listed as follows.

$\bullet$ \textbf{Model Formulation:} We propose a novel co-branding bipartite graph model that captures critical features based on the bidirectional relationships between ``initiating'' and ``target'' brands. This model not only incorporates the uncertainty of successful collaborations, but also incorporates market gains, enabling a comprehensive representation of the potential value created by different co-branding partnerships.

$\bullet$ \textbf{Algorithm Design:} To balance exploration and exploitation in an online market, we extend the MAB framework and adapt it to the co-branding scenario during the online learning phase, which emphasizes both long-term gains and short-term performance. Concurrently, in the offline optimization phase, we allocate budgets at the parent brand level, ensuring equitable investment among sub-brands and preventing over-allocation to a single sub-brand.

$\bullet$ \textbf{Theoretical Analysis:} We establish theoretical analysis and interoperability for our framework. For the online learning phase, we provide a near-optimal sub-linear regret analysis under complex uncertainty with refined exploration. In the offline optimization phase, we guarantee a $1-1/e$ approximation for the NP-hard optimal solutions under budget considerations.

$\bullet$ \textbf{Experimental Evaluation:} We perform extensive evaluations on both synthetic and real-world datasets, including curated datasets that we will publicly release. The results show that our approach achieves 39\% improvements on average over traditional methods, delivering superior short-term and long-term returns.

\section{Model and Problem Formulation}\label{sec:model}

This section presents the co-branding problem formulation, comprising both the online graph learning scenario and the offline budget allocation problem\footnote{Further discussions and extensions about the model can be found in Appendix \ref{app:model_dis}.}. For $\forall z \in\mathbb{N}_0$, let $[z]$ denote the set $\{1, 2, \dots, z\}$, with $|\cdot|$ denoting the size of a set. 

\subsection{Co-Branding Bipartite Graph Model}
\label{subsec:bipartite_graph_model}
We adopt a bipartite graph \( \cG = (\mathcal{U}, \mathcal{V}, \mathcal{E}) \) as the foundational framework to model co-branding opportunities between a parent brand and potential partner brands. This choice is motivated by the natural division in co-branding scenarios: sub-brands initiating co-branding and potential partner brands that could participate in these initiatives.
As shown in the (a) part of \cref{fig:model}, the set \( \mathcal{U} \) denotes co-branding initiators, representing a parent brand system with multiple sub-brands, where \( |\cU| = U \) is the number of sub-brands involved in co-branding. These sub-brands maintain diverse market positions and brand images within the brand architecture, enabling various co-branding initiatives. For companies without sub-brands, \( U \) reduces to 1, representing a single-brand entity. 
The set \( \mathcal{V} \)  represents target co-branding partner brands, where \(|\cV| = V \) denotes the number of available potential partner brands. These target brands allow individual assessment of each co-branding opportunity's feasibility and success probability.

\begin{figure}[!t]
    \centering
    \includegraphics[width=0.48\textwidth]{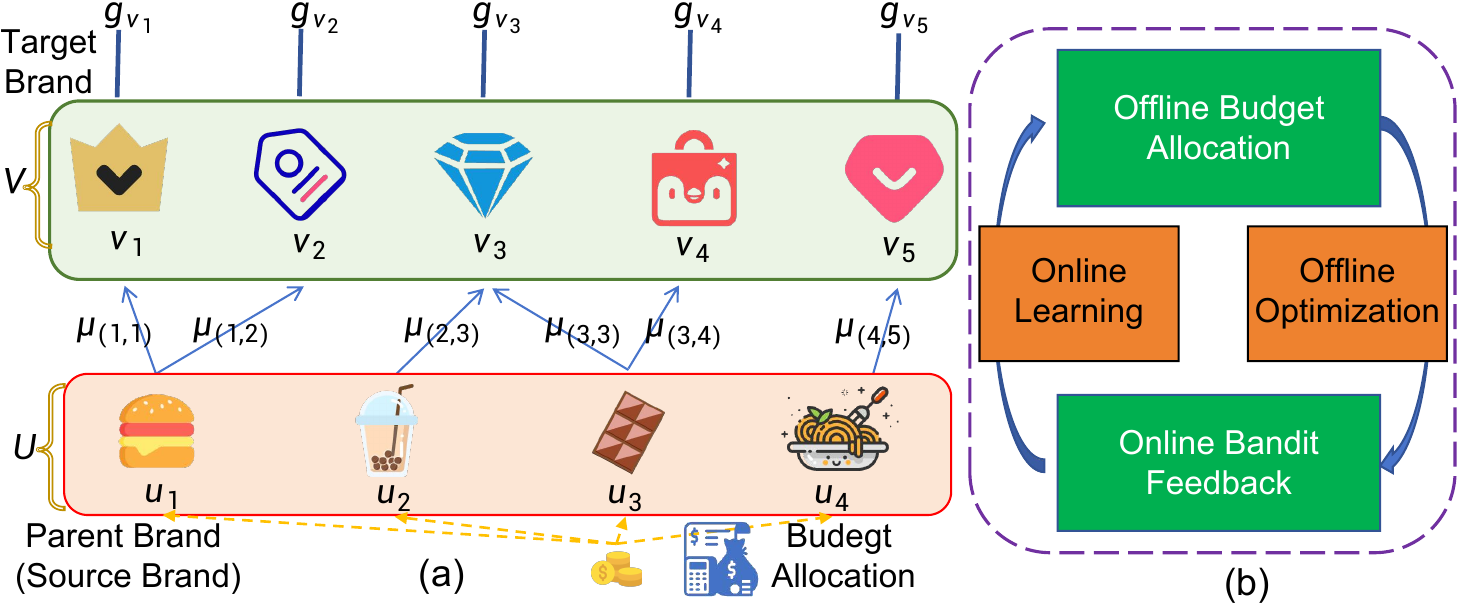}

    \caption{Co-Branding Bipartite Graph with Online-Offline Iterative Framework (with $U=4, V=5$ in this example).}
    \label{fig:model}
\end{figure}

Each edge \( e := (u, v) \in \mathcal{E} \) signifies a co-branding pair between sub-brand \( u \in \mathcal{U} \) and target partner brand \( v \in \mathcal{V} \). The weight vector \( \bmu = \{\mu_e\}_{e \in \mathcal{E}} \)  represents the probability of successful co-branding for each edge. These probabilities are influenced by multiple factors in the complex business environment, e.g., market positioning alignment, brand image compatibility, and profit-sharing arrangements. These probabilities are typically unknown a priori and require continuous online learning to update. Moreover, the probability positively correlates with the parent brand's financial budget $\bb$ (introduced in detail shortly), i.e., higher investment generally increases the likelihood of successful co-branding. Since any \(u \) belongs to the same parent company $\mathcal{U}$, there is no bidding agreement between sub-brand \( u \) and any target brand \( v \in \cV \), meaning \( v \) is not restricted to co-branding with only one specific \( u \) from \( \mathcal{U} \).
The graph \( \cG \) also incorporates the brand market gain vector \( \bg = \{g_v\}_{v \in \cV} \), which represents the revenues of selling products or services from a target brand $v$ to the entire brand company $\cU$ upon the launch of a co-branding initiative. Note that the market gain is primarily driven by consumer preferences rather than budget allocation, as a complex market environment does not necessarily imply a direct correlation between higher investment and higher returns.   Each \( g_v \in \bg \) denotes the consumer market gain from target brand $v$ to the whole parent brand \( \cU \), which is initially unknown.

\subsection{Online Feedback Mechanism}
\label{subsec:online_graph_learning}

While brands may leverage historical collaboration data to construct a static bipartite graph $\cG$, this approach is ill-suited to the ever-changing nature of market conditions. To address the inherent dynamics and uncertainties in the probabilities of successful brand partnerships and market gains, we adopt an online learning framework. This framework dynamically updates the bipartite graph \( \cG = (\mathcal{U}, \mathcal{V}, \mathcal{E}) \) based on sequentially observed outcomes.

The online learning process unfolds over \( T \) rounds, each corresponding to a co-branding season. Initially, the structure of \( \mathcal{G} \) is unclear and incomplete. During each co-branding season \( t \) (i.e., round), a budget allocation \( \boldsymbol{b}_t = (b_{t,1}, \ldots, b_{t,U}) \)  is determined based on observations from previous rounds (the budget allocation model will be introduced later).  Under this allocation, co-branding initiators \( \mathcal{U} \) engage in co-branding negotiations with target brands \( \mathcal{V} \). The resulting co-branding pairings between sets \( \mathcal{U} \) and \( \mathcal{V} \) are denoted as action \( S_t^{\bb_t} \subseteq \mathcal{E} \) for the current co-branding season. The parent company can determine $S_t$ under $\bb_t$ using various business strategies, such as prioritizing sub-brands with higher cooperation potential, distributing the budget progressively, or adopting a hybrid approach, though these heuristics are beyond our focus.
For brevity,  this is abbreviated as \( S_t \) when no ambiguity arises.

Upon executing actions on co-branding pairs, the decision-maker observes a feedback vector $\boldsymbol{X}_{t,S_t} = (X_{t,1}, \ldots, X_{t,|S_t|}) \in [0,1]^{|S_t|}$, where each \( X_{t,e,b_u} \) reflects the propensity for co-branding success for edge \( e = (u, v) \) under budget allocation $b_u$.  
Specifically, if brand \( v \) accepts a co-branding invitation from brand \( u \),  \( X_{t,e,b_u} =1 \), signaling a successful outcome. Conversely, \( X_{t,e,b_u} =0\) represents a failed invitation. 
The expectation of \( X_{t,e} \) is denoted as \( \mu_{e} \), representing the probability of successful co-branding between \( u \) and \( v \). 
 Leveraging this feedback, the decision-maker can update its estimates of the cooperation probabilities \( \bmu \), enabling it to prioritize promising brand pairings while avoiding those with uncertain prospects.
Furthermore, for each successful co-branding pair \(e= (u, v) \), where $X_{t,e}=1$, the market gain  \( Y_{t,v} \) from the target brand $v$ is observed  under the budget allocation $\bb$ for the parent brand $\cU$ in season \( t \). The expectation of \( Y_{t,v} \), denoted as  \( g_v \), represents the brand market gain, quantifying the economic benefits of the collaboration. Notably, this feedback is available only for successful co-branding actions.
We denote the vector of observed market gain feedback at season $t$ as $\boldsymbol{Y}_{t,\cV}$. By analyzing these feedback data, the decision-maker acquires valuable insights into the tangible economic benefits of partnerships, enabling more effective and informed decisions for future co-branding strategies.

As a result, when action \( S_t \) is executed under the given budget allocation, the decision-maker observes a reward  \( R_{\cG}(\bb_t) \), which is influenced by both the co-branding probabilities and the market gains under the graph $\cG$:
\begin{equation}\label{eq:reward}
    R_{\cG}(\bb_t)=\sum_{v \in \cV}\I\{\exists\, e=(u,v) \in S_t \text{ s.t. } X_{t, e,b_u}=1\}Y_{t,v}\,,
\end{equation}
where \( \mathbb{I} \{\cdot\} \) is an indicator function that evaluates to 1 if one successful co-branding action \( X_{t,e} = 1 \) exists, and 0 otherwise. For the online phase, the decision-maker needs to accurately and efficiently estimate and learn the structure of the graph $\cG$ to maximize the cumulative reward over time.

\subsection{Offline Strategic Budget Allocation}
\label{subsec:offline_budget_allocation}
As described earlier, the financial investments of the parent brand play a pivotal role in co-branding negotiations and planning for the upcoming season. Specifically, based on the online-learned market graph $\cG$, which includes the co-branding probabilities \( \bmu \) of co-branding and the corresponding market gains \( \bg \), 
  the parent brand must allocate a total budget \( B \in  \mathbb{N}
 \) (typically a large integer, e.g., in millions) across its sub-brands. The allocation is represented as \( \bb = (b_1, \ldots, b_U) \in \mathbb{N}_0^U \), where \( b_u \) denotes the budget assigned to sub-brand \( u \), subject to the constraint $\sum_{u=1}^{U} b_u \leq B$.
Subbrands \( u \in \mathcal{U} \) that do not participate in co-branding (e.g. due to poor past performance) are assigned a zero budget, i.e., \( b_u = 0 \). Furthermore, each subbrand \( u \) is assigned a predetermined budget cap \( c_u \), set by the parent brand, ensuring \( b_u \leq c_u \) for all \( u \in \mathcal{U} \). This ensures a controlled and strategic allocation of financial resources while accounting for the individual potential of each sub-brand.

Based on \cref{eq:reward} and the definitions of \( \bmu \) and \( \bg \), the expected reward  $r_{\cG}(\bb)$  for season \(t\), derived from $R_{\cG}(\bb)$ with action $S$  on $\bX_{t,S}$ and $\bY_{t,\cV}$ is expressed as:
\begin{equation}\label{eq:ex_reward}
  r_{\cG}(\bb) = \sum_{v\in \cV} g_v \left(1 - \prod_{e=(u,v) \in S} (1 - \mu_{e,b_u})\right).
\end{equation}
The offline budget allocation problem involves distributing the total budget \( B \) among \( U \) sub-brands to maximize the reward $r_{\cG}(\bb)$. To simplify marketing budget decisions, brand managers often prefer predefined strategies over continuous, fine-tuned adjustments \cite{low1998brand,pun2015note}. Accordingly, we define the set of tentative spending plans for each sub-brand \( u \) as \( \mathcal{N}_u \subseteq  \mathbb{N}_0\), where $s_u \in \mathcal{N}_u$ represents a specific tentative spending plan, still constrained by $s_u \leq c_u$. For instance, a sub-brand may offer three tiers of budget plans, such as low, medium, and high, to provide flexibility while remaining within the predefined cap.
The optimization problem is  formalized as:
\begin{equation}
\label{eq:opt_budget}
\begin{aligned}
\max_{\bb} \quad & r_{\cG}(\bb) = \mathbb{E}\left[  R_{\cG}(\bb_t) \right] \\
\text{s.t.} \quad & \bb \in \mathbb{N}_0^U, \sum_{u=1}^{U} b_u \leq B, \\
& b_u \leq c_u, \, b_u \in \cN_u, \forall u \in \cU .
\end{aligned}
\end{equation}
\Cref{eq:opt_budget} aims to optimally allocate resources among sub-brands for maximizing the total revenue, prioritizing those with higher potential returns while adhering to budget constraints. 
However, the expected reward function \( r_{\mathcal{G}}(\bb) \) is highly non-linear, rendering the problem NP-hard \cite{hochba1997approximation,wan2023bandit,dai2025variance} and extending beyond the standard knapsack problem framework \cite{agrawal2018proportional}. As solving it to optimality is computationally infeasible, we pursue an offline \textit{\(\alpha\)-approximation solution}, where \(\alpha \in (0, 1]\) denotes the approximation ratio. Specifically, the solution \( \bb \) satisfies:
$ r_{\cG}(\bb) \geq \alpha \cdot r_{\cG}(\bb^*),$
where \( \bb^* \) denotes the computationally intractable optimal budget allocation.  This approach provides a practical and efficient method to approximate the optimal solution while maintaining acceptable accuracy. Our goal is to maximize \( \alpha \), achieving a solution as close to optimal as feasible within computational constraints.

\subsection{Unified  Problem Formulation}
\label{subsec:alpha_regret}
In practical applications, complete information about the underlying graph \(\mathcal{G}\) is rarely available. The online learning problem addresses this challenge by iteratively learning the underlying graph parameters through interactions with the market environment. This decision-making process involves a critical trade-off between ``exploitation'' (leveraging historical empirical probabilities of successful co-branding and market gains) and ``exploration'' (identifying potentially better but yet undiscovered co-branding opportunities and market gains).
In the offline phase, at the beginning of each season \( t \), the decision-maker must determine the budget allocation \(\bb_t\) based on the online-learned graph $\cG$. As shown in the (b) part of \cref{fig:model}, these two components mutually reinforce each other, fostering continuous improvement in co-branding strategies.

The performance of an algorithm \( A \), which operates over both offline and online phases, is evaluated using the concept of ``\(\alpha\)-approximate regret'' \cite{li2016contextual,chen2016combinatorial,liu2025offline}. Over \( T \) rounds, the cumulative regret is defined as:
\begin{equation}
\label{eq:alpha_beta_regret}
Reg(T) = \alpha T \cdot r_{\mathcal{G}}(\bb^*) - \mathbb{E}\left[ \sum_{t=1}^T r_{\mathcal{G}}(\bb_t^A) \right],
\end{equation}
where \( \bb_t^A \) represents the budget allocation chosen by algorithm \( A \) in season \( t \). The goal is to design an algorithm \( A \) that minimizes \( Reg(T) \), ensuring that the cumulative rewards closely approximate those of the optimal allocation over \( T \) rounds.
Minimizing \( Reg(T) \) presents significant challenges. First, uncertainty in partnerships arises as leading brands often dominate resource allocation, overshadowing long-tail brands with untapped potential. Similarly, dynamic market conditions make over-reliance on historical data risky. These uncertainties jointly highlight the need for a careful balance between exploration and exploitation.  Beyond long-term performance,  early-stage failures pose a risk to stakeholder trust in online learning methods, requiring a more careful exploration strategy. Lastly, parent brands managing multiple sub-brands face intricate resource distribution challenges, where the objective shifts from optimizing individual sub-brand performance to maximizing collective benefits across the entire brand portfolio. 

\section{Algorithmic Design}\label{sec:algorithm}

As shown in Fig. \ref{fig:alg}, this section outlines our algorithm design, which seamlessly integrates both online and offline processes.

\begin{figure}[t]
    \centering
    \includegraphics[width=0.44\textwidth]{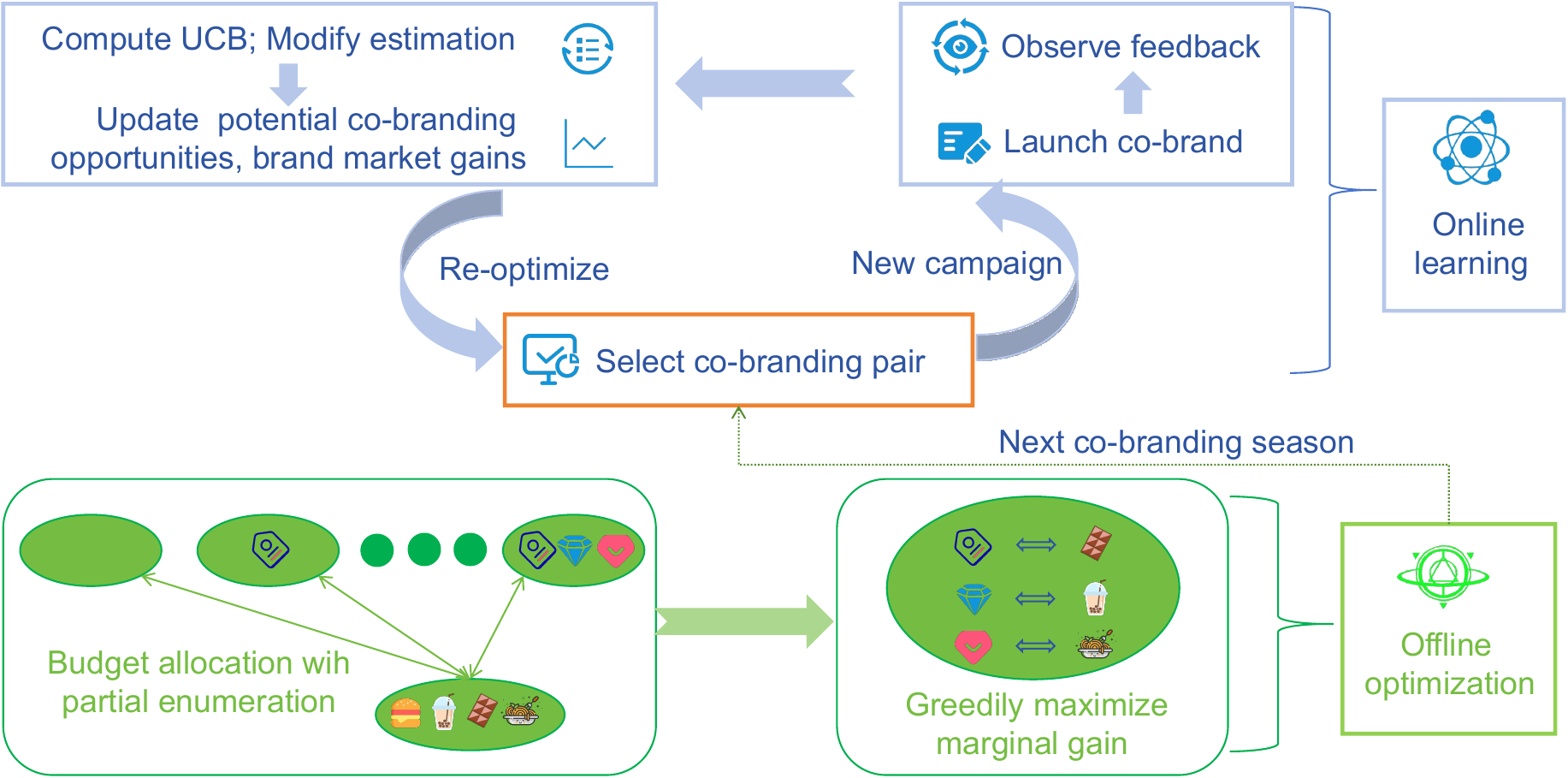}
    \caption{Workflow of Hybrid Online-Offline Algorithm: \small{It begins with the estimation of the co-branding bipartite graph, followed by offline budget allocation for selecting optimal co-branding pairs. After executing initial campaigns, market feedback is integrated to refine the estimations. These updated insights then guide the re-optimization of subsequent campaigns.}}
    \label{fig:alg}
\end{figure}

\subsection{Graph Learning via Online Feedback}\label{subsec:online-overlap}
\textbf{Exploration-Exploitation Trade-off.} In practical markets, the graph $\cG$ information structure is often \textit{partially or entirely unknown}, leading to uncertainties in identifying potential co-branding opportunities and estimating their potential gains.  A naive approach that always selects the current best-known partner based on empirical feedback risks becoming trapped in a \textit{local optimum}, as it emphasizes exploitation over exploration. To overcome this pitfall, we introduce a confidence-based multi-armed bandit (MAB) strategy that effectively balances exploration and exploitation \cite{li2025demystifying,lattimore2020bandit}.

\textbf{Algorithm Workflow.} We introduce our online algorithm in Algorithm~\ref{alg:BCUCB-T}, an adaptation of the combinatorial upper confidence bound (CUCB) algorithm tailored specifically to the co-branding context within the general combinatorial multi-armed bandit (CMAB) framework~\cite{chen2013combinatorial, chen2016combinatorial,dai2024cost}. Unlike the deterministic arm selection employed in the original CUCB algorithm~\cite{chen2016combinatorial}, co-branding introduces unique challenges: the co-branding probability $\bmu$ is influenced by the specific budget allocation rather than being fixed, and the unknown market gain $\bg$ depends on the co-branding probability, which is only observable upon success.
To address these challenges, we extend the concept of the base arm in \cite{chen2013combinatorial, chen2016combinatorial}. Specifically, we define a set of partial base arms, denoted as $\cA = \{(e, s) \mid e \in \cE, s \in \cN_u\}$, where $\mu_{e,s}$ denotes the probability of successful co-branding for each base arm  $(e, s)$. Here, $(e,s)$ corresponds to a co-branding pair $e$ combined with a tentative spending plan $s$. Furthermore, we introduce an additional set of base arms $\cA'$ to represent the remaining base arms corresponding to the unknown market gain $\bg$, which contribute to additional regret in the system. The complete set of base arms is therefore given by $\cA \cup \cA'$, allowing us to address the co-branding dynamics within our algorithmic framework.

\begin{algorithm}[t]
	\caption{Confidence-Based Online Learning for Co-Branding (CBOL)}\label{alg:BCUCB-T}	
	\begin{algorithmic}[1]
	    \REQUIRE Set of  co-branding initiators $\cU$, set of  target brands $\cV$.
\STATE Initialize \(T_{t,e,s}\), \(\hat{\mu}_{e,s}\) for each \((e,s) \in \cA\), and \(T'_v\), \(\hat{g}_v\) for each \(v \in \cA'\) using historical dataset \(\cD\).\label{line: extend over initial weight} 
		\FOR {season $t=0,1,2,...,T $}
		 \STATE    For  $(e,s) \in \cA$,
		    $\rho_{e,s} \leftarrow \cref{eq:confidence_interval}$, $\tilde{\mu}_{e,s} \leftarrow \hat{\mu}_{e,s} + \rho_{e,s}$, $\bar{\mu}_{e,s} \leftarrow \max_{j \in \cN_u,j\leq s} {\tilde{\mu}_{e,j}}$.	 
		    \label{line: over online ucb}
      \STATE   For $v \in \cA'$,
		    $\rho'_{v} \leftarrow \sqrt{\frac{6\hat{V}_{v} \log t}{T_{t,v}}} + \frac{9\log t}{T_{t,v}}$, $\hat{g}_{v} \leftarrow \hat{g}_{v} + \rho'_{v}$. 
		    \label{line: extend over online ucb for nodes}
		\STATE Budget allocation  $\bb \leftarrow$ GPE (Algorithm \ref{alg: begreedy enum}).
		\label{line: over online oracle}
		\STATE Observe co-branding intention feedback $\bm{X}_{t,S_t}$ under budget allocation $\bb$. \label{line: over online feedback}
        \STATE For each $(e,s)$ that receives feedback $X_{e,s}$,  update $T_{t,e,s}\leftarrow T_{t,e,s}+1$, $\hat{\mu}_{e,s}\leftarrow \hat{\mu}_{e,s}+(X_{e,s}-\hat{\mu}_{e,s})/T_{t,e,s}$, 
	   $\hat{V}_{e,s}\leftarrow 
\frac{T_{t,e,s}-1}{T_{t,e,s}}\left(\hat{V}_{e,s} + \frac{1}{T_{t,e,s}} \left(\hat{\mu}_{e,s} - X_{t,e,s}\right)^2\right)$. 
	    \label{line: over online update}
		\STATE {For any successful  co-branding pair $e \in S_t$ with $X_{t,e,s}=1$, observe market gain $\bY_{t,\cV}$ and update $T'_{t,v} \leftarrow T'_{t,v} + 1$,  $\hat{g}_v \leftarrow \hat{g}_v + (Y_v - \hat{g}_v )/T'_{t,v}$}, 
	   $\hat{V}'_{v}\leftarrow \frac{T_{t,v}'-1}{T_{t,v}'}\left(\hat{V}_{v}' + \frac{1}{T_{t,v}'} \left(\hat{g}_{v} - Y_{t,v}\right)^2\right)$.\label{line: update weight} 
		\ENDFOR
		\end{algorithmic}
\end{algorithm}

We estimate the co-branding probability $\mu_{e,s}$ using the empirical mean $\hat{\mu}_{e,s}$, where $T_{t,e,s}$ record the number of times the base arm $(e,s) \in \cA$ has been selected up to season $t$; Similarly, $\hat{g}_v$ represents the empirical mean of $v$'s market gain, with $T'_v$ recording the number of times the target brand $v \in \cA'$ has been selected (line~\ref{line: extend over initial weight}).   Algorithm~\ref{alg:BCUCB-T} leverages confidence bounds to balance exploitation and exploration. Specifically, for each season $t$, the confidence radius $\rho_{e,s}$, computed in line~\ref{line: over online ucb}, regulates the level of exploration for potentially better co-branding pairs. This radius is larger for underexplored arms $(e,s)$ (i.e., when $T_{t,e,s}$ is small), thereby incentivizing exploration.
Similarly, in line~\ref{line: extend over online ucb for nodes}, $\bar{g}_v$ is the optimistic estimate of the market gain $g_v$ with confidence radius $\rho'_v$. Co-branding feedback is collected after executing the action $S_t$ under budget allocation $\bb$, as determined by Algorithm~\ref{alg: begreedy enum} (introduced subsequently). In line~\ref{line: over online update}, the corresponding estimations are updated based on the online bandit feedback regarding co-branding outcomes $\bm{X}_{t,S_t}$. Finally, the unknown market gain is refined in line~\ref{line: update weight} following the implementation of the co-branding campaign.

\textbf{Reduce Redundant Exploration.} Note that the original confidence radius in \cite{chen2016combinatorial,wang2017improving} relies solely on the empirical mean, neglecting variance and resulting in a loose confidence radius for each arm. This can lead to unnecessary exploration. To address this, inspired by \cite{liu2022batch}, we adopt a stronger Bernstein-type concentration bound \cite{lattimore2020bandit}, which incorporates the empirical variance $\hat{V}_{t-1, i}$ to construct an improved confidence radius (line \ref{line: over online ucb}):
\begin{equation}\label{eq:confidence_interval}
    \rho_{e,s}=\sqrt{\frac{6\hat{V}_{e,s} \log t}{T_{t,e,s}}} + \frac{9\log t}{T_{t,e,s}}.
\end{equation}
In the first term of \cref{eq:confidence_interval}, $\hat{V}_{e,s}$ serves as an approximation of the true variance $V_{e,s}$. Since $V_{e,s} \leq (1 - \mu_{e,s}) \mu_{e,s}$, the estimation of $\mu_{e,s}$ becomes more precise when $\mu_{e,s}$ approaches boundary values.
The second term of \cref{eq:confidence_interval} compensates for using the empirical variance $\hat{V}_{e,s}$ instead of $V_{e,s}$. A similar tighter confidence radius strategy is applied to estimate the unknown market gains (line \ref{line: extend over online ucb for nodes}).

\begin{algorithm}[!t]
\caption{Greedy Partial Enumeration for Budget Optimization (GPE)}\label{alg: begreedy enum}
\begin{algorithmic}[1]
		\REQUIRE Co-Branding graph $\cG$, total budget $B$, budget cap $c_u$, tentative spending plans $\cN_u, u\in\cU$, operational constraint $K$.

			\STATE Initialize $\bb_{max} \leftarrow \bm{0}$. 
			\STATE $\cB \leftarrow \{\bb=(b_1, ..., b_U)|0\le b_u \le c_u, b_u \in \cN_u, \sum_{u \in \cU}b_u \le B, \sum_{u \in \cU}\mathbb{I}\{b_u > 0 \} \le K\}$.\label{line: enumeration}
		\FOR{$\bb \in \cB$}
		    \STATE $B'\leftarrow B-\sum_{u \in \cU}b_u$.
		    \STATE Let $\cQ \leftarrow \{(u,s_u) \, | \, u \in \cU, s_u \in \cN_u, 1\le s_u \le  c_u-b_u\}$.
		    \WHILE {$B' > 0$ and $Q \neq \emptyset$ \label{line: enum while begin}}
		        \STATE $(u^*,s^*) \leftarrow \argmax_{(u,s)\in \cQ} \delta(u,s,\bb)/s $. \label{eq: enum largest margin}  \label{line: largest margin} 
		        \IF {$s^* \le B'$}
		            \STATE $s_{u^*} \leftarrow s_{u^*} + s^*$, $B' \leftarrow B' - s^*$.\label{line: deduct}
		            \STATE Adjust all pairs $(u^*,s) \in \cQ$ to $(u^*, s - s^*)$.
		            \STATE Remove all pairs $(u^*,s) \in \cQ$ such that $s \le 0$. \label{line: remove}
		        \ELSE \STATE Remove $(u^*, s^*)$ from $\cQ$.
		            \label{line: enum remove}
		        \ENDIF
	    	\ENDWHILE  
	    	\STATE \textbf{if} $r_{\cG}( \bb) > r_{\cG}( \bm{b_{max}})$, \textbf{then} $\bb_{max} \leftarrow \bb$.\label{line: replace}
		\ENDFOR
		\end{algorithmic}
\end{algorithm}

To further reduce unnecessary exploration, we leverage historical dataset $\cD$ consisting of past season data with a scale of \(D\) to initialize observed potential co-branding probabilities \(\bmu\), brand market gains \(\bg\), and selecting counts \(T_{t,e,s}\) (and \(T'_{t,v}\)) (line \ref{line: extend over initial weight}). Note that naively incorporating historical data can degrade short-term performance, such as over-exploring low-reward arms or under-exploring high-reward ones, thereby incurring high regret \cite{xu2021dual}. Thus, Algorithm \ref{alg:BCUCB-T} excludes historical arm selections when computing the confidence radius \(\rho_{e,s}\) (and \(\rho_v'\)), relying only on online feedback. This balances long-term benefits of historical data with reduced short-term exploration.

\textbf{Non-Decreasing UCB Values.} In line~\ref{line: over online ucb}, we compute the upper confidence bound (UCB) value  \(\tilde{\mu}_{e,s}\). Although higher spending typically increases (or at least does not decrease) collaboration probability in real-world co-branding scenarios, random exploration can cause \(\tilde{\mu}_{e,s}\) to paradoxically drop with larger \(s\). To reflect realistic spending–collaboration trends, we ensure non-decreasing UCBs by setting each \(\tilde{\mu}_{e,s}\) to the maximum of all \(\tilde{\mu}_{e,j}\) with \(j \le s\).

\subsection{Budget Optimization via Offline Planning}

\textbf{Submodular Property Basis.}  To address the budget allocation problem defined in \cref{eq:opt_budget}, we exploit the \textit{submodular property of the reward function}, which intuitively reflects the diminishing marginal returns as the budget increases \cite{liu2021multi,alon2012optimizing} (Formal proof of this property is in Section \ref{subapp:property}). While increasing the budget for a single sub-brand $u$ within a co-branding strategy may not directly exhibit a diminishing marginal effect, it is noted that the overall marginal gain for the entire parent brand $\cU$ decreases as the total budget is distributed. This occurs because the available budget for other sub-brands within the same parent brand becomes limited as more budget is allocated to $u$. We now formally introduce this definition. For any $\bx, \by \in \mathbb{N}_0^U$, we denote $\bx \vee \by, \bx \wedge \by \in \mathbb{N}_0^U $ as the coordinate-wise maximum and minimum of these two vectors: $(\bx \vee \by)_i=\max\{x_i, y_i\}, (\bx \wedge \by)_i=\min\{x_i, y_i\}$.
A function $f: \mathbb{N}_0^U \rightarrow \bbR$ is defined as \textit{submodular} over the integer lattice $\mathbb{N}_0^U$ if it satisfies the following inequality for any $\bx, \by \in \mathbb{N}_0^U$: $f(\bx \vee \by) + f(\bx \wedge \by) \le f(\bx) + f(\by)$.
Let $\bch_i = (0, ..., 1, ..., 0)$ represent the one-hot vector with the $i$-th equal to 1 and all other elements equal to 0.

\textbf{Refining Approximation Ratio.} Building on strategies that combine greedy algorithms with submodular property \cite{khuller1999budgeted, alon2012optimizing, liu2021multi}, which mainly give \(\alpha = \frac{1}{2}(1-e^{-1}) \approx 0.316\)-approximation and \(\alpha = (1-e^{-\eta}) \approx 0.357\)-approximation (where \(\eta\) satisfies \(e^\eta = 2 - \eta\)), we integrate partial enumeration methods to further refine the approximation ratio. In the context of co-branding, the high cost of creating co-brands for parent brands requires significant preliminary efforts, such as extensive market research and the costs associated with creating co-branding materials. Unlike traditional problems focused on reducing time complexity, the need for such upfront investments in co-branding makes it essential to optimize the approximation ratio without sacrificing the quality of the outcomes. Leveraging these insights, we focus on improving the approximation ratio for better co-branding results, while ensuring that the approach remains computationally tractable within polynomial time.
The detailed procedure is presented in Algorithm \ref{alg: begreedy enum}.

\textbf{Algorithm Workflow.} Algorithm \ref{alg: begreedy enum} constructs an initial solution set \( \mathcal{B} \), which stores all partial solutions that allocate budgets to at most \( K \) sub-brands (line~\ref{line: enumeration}). Each initial solution corresponds to a potential final solution, from which we must select the optimal one. The parameter \( K \) serves as an operational constraint, and as \( K \) increases, the number of possible initial solutions grows, potentially exponentially. While increasing \( K \) can improve the approximation ratio by permitting more flexible allocations (e.g., distributing budgets evenly or dedicating the entire budget \( c_u \) is allocated to a single sub-brand \( u \), it also significantly increases the computational complexity. Based on the theoretical analysis in \Cref{subapp:offline}, \( K = 3 \) provides a relatively reasonable trade-off between runtime efficiency and approximation quality. For any partial solution, Algorithm \ref{alg: begreedy enum} maintains a set $\cQ$, where any pair $(u,s) \in \cQ$ represents a tentative plan to allocate an additional budget of $s$ to sub-brand $u$.
The algorithm follows a \textit{greedy iteration strategy}: a while loop (line~\ref{line: enum while begin}-\ref{line: remove}) that iteratively selects the pair $(u,s)$ from $\cQ$ such that that maximizes the \textit{per-unit marginal gain} $\delta(u,s, \bb)$. Specifically, the per-unit marginal gain is defined as $\delta(u,s, \bb) = (r_{\cG}( \bb+s\bm{\chi}_u) - r_{\cG}( 
\bb))/s$,  where \(r_{\cG}\) is the reward function, \(\bb\) is the current budget allocation, and \(\bm{\chi}_u\) represents the one-hot allocation vector for sub-brand \(u\).
If \(s\) is within the remaining budget, it is allocated to sub-brand \(u\); otherwise, the algorithm proceeds to the next iteration.
Note that Algorithm \ref{alg: begreedy enum} ensures monotonicity of the input due to the maximum operator (line~\ref{line: over online ucb} in Algorithm \ref{alg:BCUCB-T})\footnote{A function \(f\) is monotone if  \(f(\bx) \leq f(\by)\) for all vectors \(\bx \leq \by\), where  $\bx \ge \by$ means $x_i \ge y_i$ for all elements $i$.}. 

\textbf{Integration with Online Learning.}  After budget allocation, sub-brands use the learned co-branding graph \( \mathcal{G} \) to engage with target brands for executing co-branding campaigns.  If a sub-brand \( u \) has multiple successful co-branding partners from the target brand set $\cV$, it allocates its budget \( b_u \) in a balanced manner, thus ensuring the exploration of diverse co-branding opportunities without over-investing in a single target brand. Alternatively, the budget can be proportionally allocated based on market returns, target brand performance, and strategic priorities, which depend on the brand's economic context and are beyond the primary focus of our research
Executing a co-branding campaign involves forming partnerships, deploying marketing resources, and monitoring performance. The corresponding feedback is used to update the estimates in future co-branding seasons, as described in Algorithm \ref{alg:BCUCB-T}.

\section{Theoretical Analysis}\label{sec:theorem}
 In this section, we provide a detailed theoretical analysis of the performance of our proposed algorithms. The key results are summarized here, with formal proofs presented in \Cref{app: proof}.

\subsection{Monotone Submodular Property}\label{subapp:property}

 \begin{lemma}[Graph-based Lattice Submodularity]\label{thm: lattice submodular}
For any graph $\cG$, $r_{\cG}(\cdot): \mathbb{N}_0^U \rightarrow \R$ is monotone and submodular, i.e.,
	$r_{\cG}(\bm{x \wedge y}) + r_{\cG}(\bm{x \vee y}) \le r_{\cG}(\bm{x}) + r_{\cG}(\bm{y})$ for any $\bm{x},\bm{y} \in \mathbb{N}_0^U$, and $r_{\cG}(\bm{x}) \le r_{\cG}(\bm{y})$ if $\bm{x} \le \bm{y}$.
\end{lemma}
\begin{proof}
(\textbf{Monotone}) Recall the definition of expected reward $ r_{\cG}(\bb) = \sum_{v\in \cV} g_v \left(1 - \prod_{e=(u,v) \in S} (1 - \mu_{e,b_u})\right)$ and  $\bch_i $ represents the one-hot vector with the $i$-th equal to 1 and all other elements equal to 0. 
Since the fact that the co-branding willingness $\bmu$ does not decrease as budget increases, we have
\begin{align*}
&r_{\cG}( \bm{x}+\bch_j)-r_{\cG}( \bm{x})\\
&=\sum_{v\in\cV}g_v((\prod_{e=(j,v),i\neq j}\left(1-\mu_{(i,v),x_i}))\left(\mu_{e,x_j+1}-\mu_{e,x_j}\right)\right) \ge 0,
\end{align*}
for any $\bx \in \mathbb{N}_0^U, j \in \cU$.
Then we can use the above inequality repeatedly to show that $r_{\cG}( \bm{x}) \le r_{\cG}( \bm{y})$ when $\bm{x} \le \bm{y}$.

(\textbf{Submodular}) For submodular property of $r_{\cG}(\bb)$, we will first prove that $(1 - \prod_{e=(u,v) \in S} (1 - \mu_{e,b_u}))$ is submodular for any $v \in \cV$.

Denote \(h(\bm{x}) = \prod_{e=(u,v) \in S} (1 - \mu_{e,x_u})\) for simplicity. We aim to show that $h(\bm{x})$ satisfies the following inequality for any \(\bm{x} \in \mathbb{N}_0^U\), \(v \in \mathcal{V}\), and distinct \(l,j \in \mathcal{U}\) :
\begin{equation}\label{eq:submoular_tem}
h(\bm{x}) - h(\bm{x} + \bch_l) \geq h(\bm{x} + \bch_j) - h(\bm{x} + \bch_l + \bch_j).
\end{equation}
The proof is as follows. Consider the left-hand side in \Cref{eq:submoular_tem}: 
\begin{align*}
&
\quad h(\bm{x}) - h(\bm{x} + \bch_l) = \\
& \left( \prod_{e \in S \setminus \{(l,v), (j,v)\}} (1 - \mu_{e,x_u}) \right) (1 - \mu_{(j,v),x_j}) (\mu_{(l,v),x_l+1} - \mu_{(l,v),x_l}).
\end{align*}
Similarly, the right-hand side in \Cref{eq:submoular_tem} is:
\begin{align*}
& 
\quad h(\bm{x} + \bch_j) - h(\bm{x} + \bch_l + \bch_j) = \\
&\left( \prod_{e \in S \setminus \{(l,v), (j,v)\}} (1 - \mu_{e,x_u}) \right) (1 - \mu_{(j,v),x_j+1}) (\mu_{(l,v),x_l+1} - \mu_{(l,v),x_l}).
\end{align*} 
The common factor in both expressions is:
\[ \left( \prod_{e=(u,v) \in S, e \neq (l,v), (j,v)} (1 - \mu_{e,x_u}) \right) (\mu_{(l,v),x_l+1} - \mu_{(l,v),x_l}). \]
Thus,  \Cref{eq:submoular_tem} reduces to comparing:
$ (1 - \mu_{(j,v),x_j}) \geq (1 - \mu_{(j,v),x_j+1}). $
Since \( \mu_{e,b_u} \) is non-decreasing in \( b_u \) (i.e., \( \mu_{(j,v),x_j+1} \geq \mu_{(j,v),x_j} \)), it follows that: \(1 - \mu_{(l,v),x_l} \geq 1 - \mu_{(j,v),x_j+1}\).

We invoke the following lemma to formalize the submodularity:
\begin{lemma}[Lattice Submodularity Condition]\label{lem: lattice submodular}
A function $f: \mathbb{N}_0^U \to \mathbb{R}$ is submodular if and only if
$f(\bm{x}+\bm{\chi}_i) - f(\bm{x}) \ge f(\bm{x} + \bm{\chi}_j +\bm{\chi}_i) - f(\bm{x} + \bm{\chi}_j),$
for any $\bm{x} \in \mathbb{N}_0^U$ and $i\neq j$.
\end{lemma}

Then Lemma~\ref{lem: lattice submodular} (proved in \Cref{subapp:proof_lem}) and the holding of \Cref{eq:submoular_tem}  confirm the submodularity of $h(\bm{x}) = (1 - \prod_{e=(u,v) \in S} (1 - \mu_{e,b_u}))$ for any \( v \in \mathcal{V} \). With the explicit formula of the reward function given by Eq.~(\ref{eq:ex_reward}), it shows the reward function is a weighted sum over all \( v \in \mathcal{V} \) with \( g_v \). Since a non-negative weighted sum of submodular functions is submodular, \(r_{\cG}(\bb) \) is submodular.
\end{proof}

\subsection{Online Learning}

\begin{theorem}[Regret Bound] \label{thm: regret}
Algorithm~\ref{alg:BCUCB-T} has the following instance-independent sub-linear regret bound   $O(V\sqrt{(NU+1)T\log T} $ $+ \log (UVT+VT) \log T)$, where $N = \max_{u \in \cU}|\cN_u|$.
\end{theorem}

\textbf{Remark 1.} We also provide an instance-dependent regret analysis for Algorithm~\ref{alg:BCUCB-T}, as detailed in \Cref{subapp:dependent}.
  Compared to prior work \cite{chen2013combinatorial,liu2021multi}, our base arms expand from $\cA$ to $\cA \bigcup \cA'$, introducing higher uncertainty and additional regret relative to scenarios with known market gains. Through careful analysis, we incorporate the error term arising from the unknown and uncertain market gain into the base arms. In standard CMAB formulations \cite{liu2022batch,chen2016combinatorial,dai2025variance}, the set $\tau$ is typically defined as the set of triggered arms. In our case, this would imply a size of $NUV+V$.  However, since only the arms in  $\{(e,s)\in \cA| b_u=s\}$, representing actual spending decisions, and their corresponding target brands in $\cA'$ affect the rewards, we redefine $\tau$ as the set of observed base arms within a determined budget. This modification reduces the size of $\tau$ to $UV+V$, thereby improving the leading $\ln T$ term in the regret by a factor of $ O(\frac{U+1}{NU+1})$.   Additionally, we treat all past selections of a base arm as a single pull, using the average reward from historical dataset $\cD$, which bounds the regret from $\cD$ by a time-independent constant.

\subsection{Offline Optimization}
\begin{theorem}[Approximation]\label{thm:offline}
Algorithm~\ref{alg: begreedy enum} achieves a $(1-1/e)$-approximate solution, i.e., the approximation ratio is $\alpha=1-1/e$.
\end{theorem}

\textbf{Remark 2.} 
By combining partial enumeration techniques with a greedy algorithm, we exploit the unique properties and requirements of the co-branding optimization problem to achieve a $(1-1/e)$ approximation ratio.
This is the best possible solution in polynomial time unless $P=NP$ \cite{alon2012optimizing,khuller1999budgeted}.  Specifically, for any partial enumeration $\bb \in \cB$, the while loop contains $B$ iterations in the worst case. The size of the queue $\cQ$ is $O(N^2UV)$, and line~\ref{line: largest margin} requires $O(V)$ to calculate $\delta(u,s,\bb)$ through pre-computation and update. As a result, the time complexity of any partial enumeration in Algorithm~\ref{alg: begreedy enum}  is $O(BN^2UV^2)$,  which scales linearly with $U$ and quadratically with $V$, avoiding exponential growth. 
Note that in practice, budget allocations are typically limited to a finite set of candidate plans, $\cN_u$, which significantly reduces the scale of the problem. For instance, when only three candidate allocation tiers (e.g., low, medium, high) are provided, a single sub-brand would have at most three possible budget allocations, further optimizing computational efficiency.
Additionally, for large graphs, the parameter 
$K$, which balances time and performance, can be set to 1 to further accelerate processing.

\section{Experiments}\label{sec:experiment}

In this section, we conduct extensive experiments.
Specifically, we seek to address the following research questions (RQs):

\textbf{RQ1}: In high-uncertainty co-branding recommendations, can our online learning algorithm achieve superior both long-term \& short-term revenue performance compared to existing methods?

\textbf{RQ2}: Can our unified offline budget allocation strategy across multiple sub-brands significantly enhance overall revenue?

\textbf{RQ3}: Across diverse scenarios (e.g., different budget caps, seasons, and tentative plans) can the proposed framework maintain stable performance improvements?

\subsection{Implementation Details}\label{sec:experiment-settings}

All experiments were conducted on a MacBook Pro with an M3 8-core CPU. To ensure statistical reliability, each experiment was repeated 10 times, and we reported the average results along with confidence intervals derived from these repetitions.

\textbf{Baseline Algorithms.}
  We select the following algorithms from existing studies as baselines for comparison with \Cref{alg:BCUCB-T}:
  1) EMP \cite{lattimore2020bandit}, which allocates budgets and co-brands according to the empirical mean without exploration. 2) $\epsilon$-Greedy \cite{singh2021reinforcement}, which distributes the budget uniformly, randomly selecting co-branding pairs with probability $\epsilon$, or otherwise allocates budgets according to empirical mean   ($\epsilon$ set to 0.1 in all tests). 3) Thompson Sampling (TS)~\cite{wang2018thompson}, which uses Beta distribution $\text{Beta}(\beta_1,\beta_2)$ ($\beta_1=\beta_2=1$ initially) as a prior for estimating the co-branding bipartite graph model. 4) CUCB~\cite{chen2016combinatorial,wang2017improving}, a combinatorial bandit algorithm that does not mitigate redundant exploration. In addition, we also specifically compare \Cref{alg: begreedy enum} with the following offline algorithms: GBO \cite{alon2012optimizing}, a greedy algorithm for budget optimization without partial enumeration, and the proportional methods PROP-S and PROP-W \cite{liu2021multi}, which allocate budgets for co-branding based on the number of sub-brands and the market gains, respectively.

\textbf{Data Generation and Pre-processing.}
Our synthetic dataset follows prior studies~\cite{ wang2017improving,liu2022batch} and is structured as a \(10 \times 60\) matrix, where rows represent ten sub-brands under a parent brand initiating co-branding, and columns correspond to sixty potential target brands. To evaluate the co-branding strategy under varying budgets, we synthesize collaboration probabilities between sub-brands \(\mathcal{U}\) and target brands \(\mathcal{V}\), considering both inherent brand compatibility and budget impact. Specifically, we generate a compatibility parameter \(\nu_{(u,v)} \sim U(-1,1)\) for each pair \((u,v)\), representing baseline affinity, where negative values indicate minimal collaboration intent without budget support. Given a sub-brand \(u\) with budget \(s_u\) (budget allocation scheme detailed later), the probability of successful collaboration is defined as \(\mu_{(u,v), s_u} = \sigma(\nu_{(u,v)} + s_u)\), where \(\sigma(x) = 1/(1 + e^{-x})\) \cite{faury2020improved,liu2024combinatorial}. As \(s_u\) increases, \(\nu_{(u,v)} + s_u\) rises, leading to a higher \(\sigma(\cdot)\), this ensures that larger budgets enhance partnership success \cite{johnson2014logistic}. To reflect market uncertainty where higher investment does not always yield higher returns, we introduce a stochastic factor \(g_v \sim U(0,1)\) for each \(v \in \mathcal{V}\).

We further conduct experiments on real-world datasets.  As there are no publicly available datasets for co-branding evaluation, we developed a Python-based web crawler using keywords (e.g., ``co-branding campaigns,'' ``brand collaboration cases,'' ``cross-category partnerships'') to collect 3,500 co-branding cases up to 2024 from SocialBeta \cite{SocialBeta}, a platform dedicated to brand marketing insights. We then categorized these cases by the types of brands involved, constructing three real-world datasets corresponding to the initiators of the co-branding efforts, namely “\textit{diet, apparel, and IP-themed}” brands. Each dataset comprises multiple categories of potential partner brands and their historical co-branding frequencies; due to commercial protection policies, these categories represent general groupings rather than specific brands. We then establish a baseline affinity \( \nu_{(u,v)} \) for each pair of sub-brand \(u\) and target brand \(v\) by calculating the proportion of co-branding occurrences relative to the total cases, and we incorporate standardized market values from a cloud-native data catalog platform, dataworld\footnote{https://data.world/, 2025}, to quantify the market gain \( \bg \). In the Diet dataset, the set of initiating brands $\cU$ spans 10 categories totaling 269 entities (with category sizes of 32, 18, 22, 38, 25, 36, 21, 19, 28, and 30) and the set of target brands  $\cV$ covers 12 categories totaling 608 entities (64, 56, 52, 44, 58, 49, 72, 40, 38, 42, 28, and 65). The Apparel dataset has $\cU$ split into 12 categories totaling 192 entities (22, 14, 16, 24, 18, 20, 15, 12, 9, 14, 11, 17) and $\cV$ split into 10 categories totaling 471 entities (64, 56, 52, 44, 58, 49, 40, 38, 42, 28). The IP-themed dataset has $\cU$ with 8 categories (17, 24, 26, 12, 18, 27, 23, 14) totaling 161 entities and $\cV$ with 8 categories (64, 56, 72, 40, 38, 42, 28, 65) totaling 405 entities.
We then apply the same procedures as used for the synthetic dataset.
For more details and visualization on the three real-world datasets, please refer to \Cref{app:experiment}.

\textbf{Other Setups.} 
We set the historical dataset \( \mathcal{D} \) to cover the past \( D = 50 \) seasons, similar to \cite{xu2021dual}. The initial values of \( \bmu, \bg, T_{t,e,s}, T'_{t,v} \) are determined by averaging over this dataset. For the synthetic dataset and the three real-world datasets, we define the base budget \( B_0 \) as one-hundredth of the market revenue observed in each dataset's respective historical dataset \( \mathcal{D} \). The total budget \( B \) is then set to be ten times \( B_0 \) to ensure the profitability of the co-branding campaign. Similarly, the budget cap for any sub-brand \( u \), denoted as \( c_u \), is set proportional to its market share, which is defined as the ratio of brand \( u \)'s revenue to the total market revenue in its historical dataset \( \mathcal{D} \), scaled by twice the total budget \( B \). Since the total budget caps exceed $B$, some brands may not receive their full allocation \( c_u \). 
The budget allocation is chosen from a predefined set of tentative spending plans \( \cN_u \), consisting of three candidate levels: low, medium, and high. Specifically, each sub-brand $u$ can be allocated one of three budget values: \( \left\lfloor\frac{1}{3}c_u \right\rfloor \), \( \left\lfloor\frac{2}{3}c_u \right\rfloor \), and \( c_u \). We determine \( S \) by sorting the predefined levels in \( \cN_u \) in increasing order under a given \( \bb \).
Additionally, we conduct exploratory experiments on different  \( B \), \( c_u \), and \( \cN_u \).

\subsection{Evaluation Results}

\textbf{Co-branding Online Performance (RQ1).} We test \Cref{alg:BCUCB-T} on the received market gain of co-branding pairs, as in  \cref{eq:reward}.
The evaluation spans \(T = 2,000\) rounds, with results shown in \Cref{fig:all_reward}.
Across all datasets, our algorithm consistently outperforms the baselines, achieving the fastest convergence and improving the online received revenue by 12\% to 73\% compared to the baselines.

\begin{figure}[!t]
    \centering
    \includegraphics[width=8.5cm]{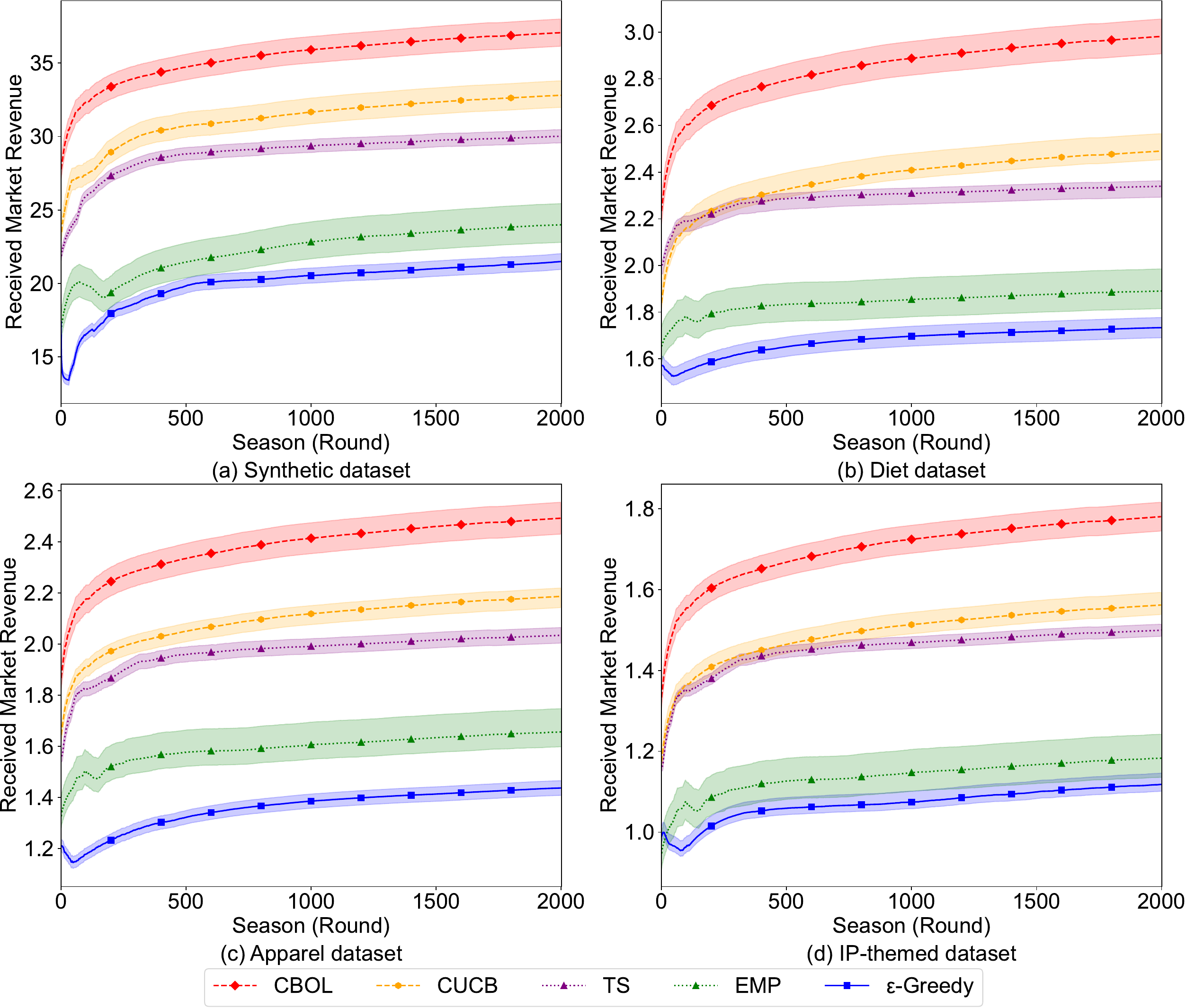}
    \vspace{-0.25in}
    \caption{Comparison of average received market revenue.}\label{fig:all_reward}
    \Description{Comparison of average reward.}
    \vspace{-0.05in}
\end{figure}

\textbf{Offline Budget Allocation (RQ2).} We evaluate \Cref{alg: begreedy enum} across different total budgets $B$ to assess whether our proposed GPE algorithm enhances performance across multiple sub-brands. As shown in \Cref{fig:offline}, GPE outperforms three baselines GPO, PROP-S, and PROP-W in revenue acquisition within the same $\cG$, achieving improvements of 13\%, 42\%, and 29\%, respectively, showing its effectiveness from a holistic parent brand perspective.

\textbf{Tradeoff Between Performance and Computational Cost (RQ2 \& RQ3).} 
 We test the impact of different $K$ (see line \ref{line: enumeration} of GPE). As shown in \Cref{fig:K}, when \( K \) increases from 1 to 3, the number of possible initial solutions expands, leading to a significant increase in the co-branding revenue. However, when \( K \) reaches or exceeds 4, the revenue gain becomes relatively marginal. On the other hand, \Cref{tab:example_table} shows that the average running time of different budget settings for \( K = 4 \) and \( K = 5 \) increases by more than a thousand times compared to \( K = 1 \)  which aligns with our theoretical analysis.

  \begin{figure}[th]
    \centering
    \includegraphics[width=0.44\textwidth]{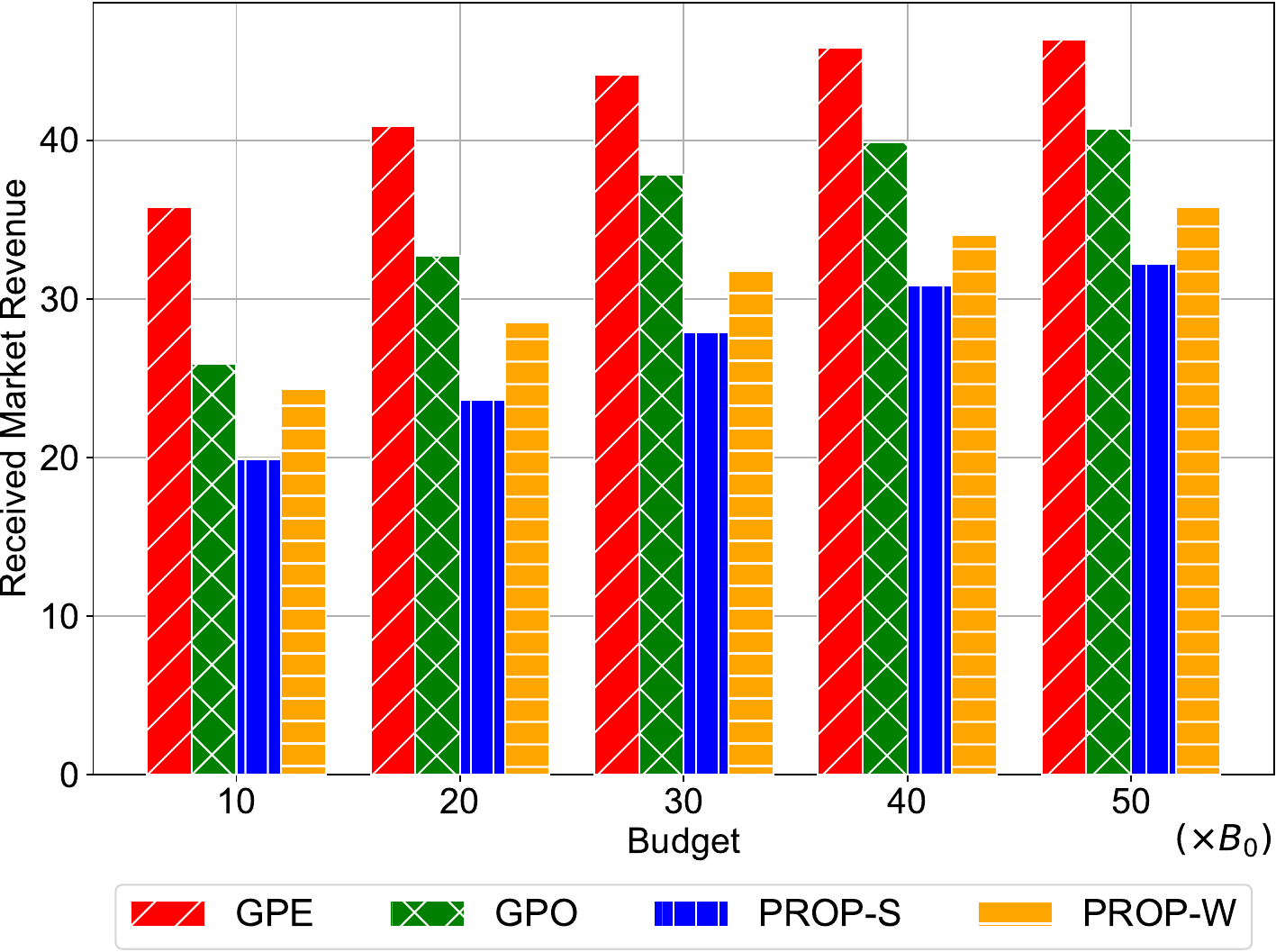}
    \vspace{-0.1in}
    \caption{Comparison of offline optimization across different total budgets on Synthetic dataset.}\label{fig:offline}
    \Description{Comparison of average reward.}
    \vspace{-0.05in}
\end{figure}

\begin{figure}[!thb]
    \centering
 \begin{subfigure}[b]{0.234\textwidth}
        \includegraphics[width=\textwidth]{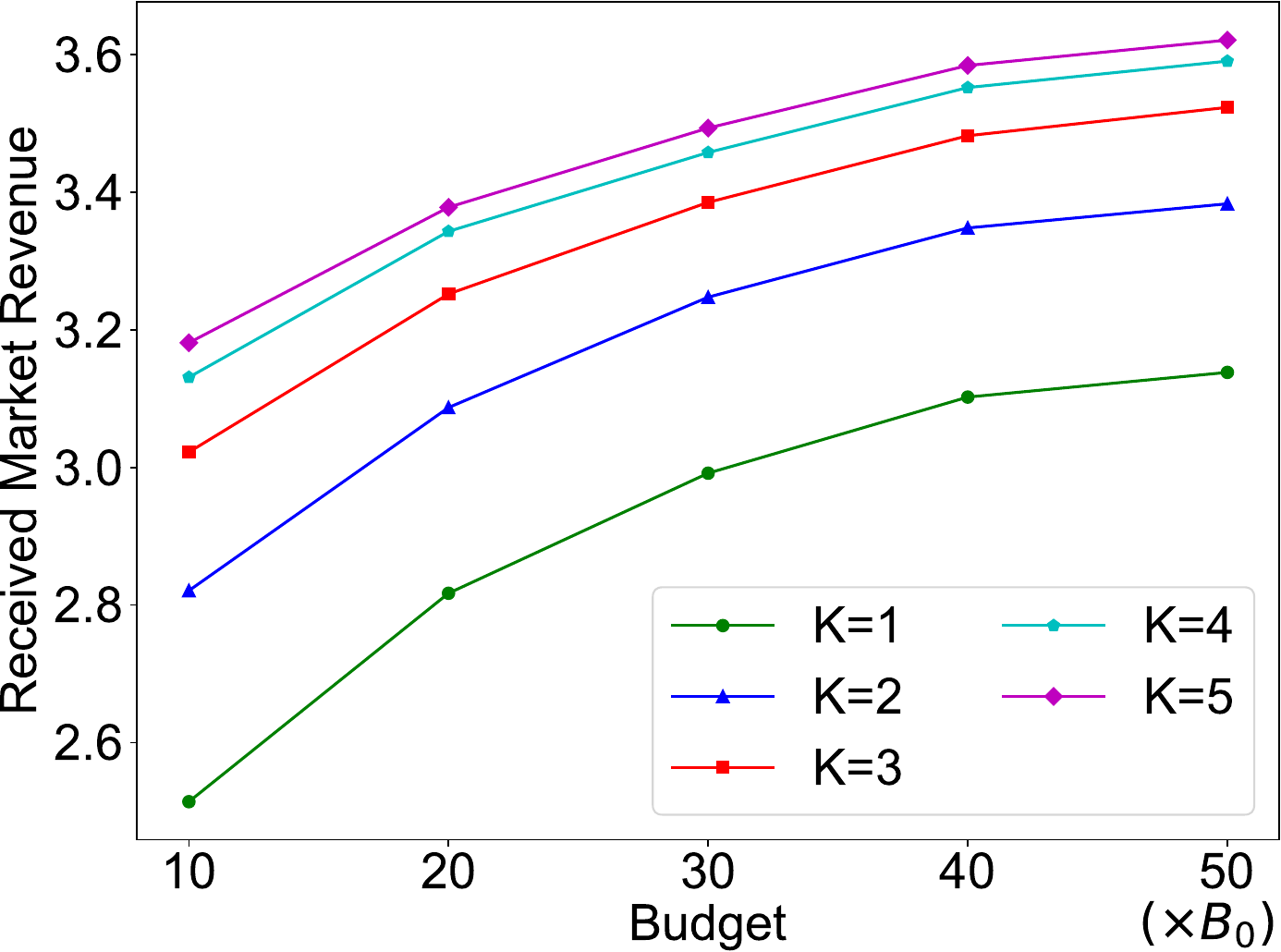}
        \caption{Diet Dataset} 
    \end{subfigure}
    \begin{subfigure}[b]{0.234\textwidth}
        \includegraphics[width=\textwidth]{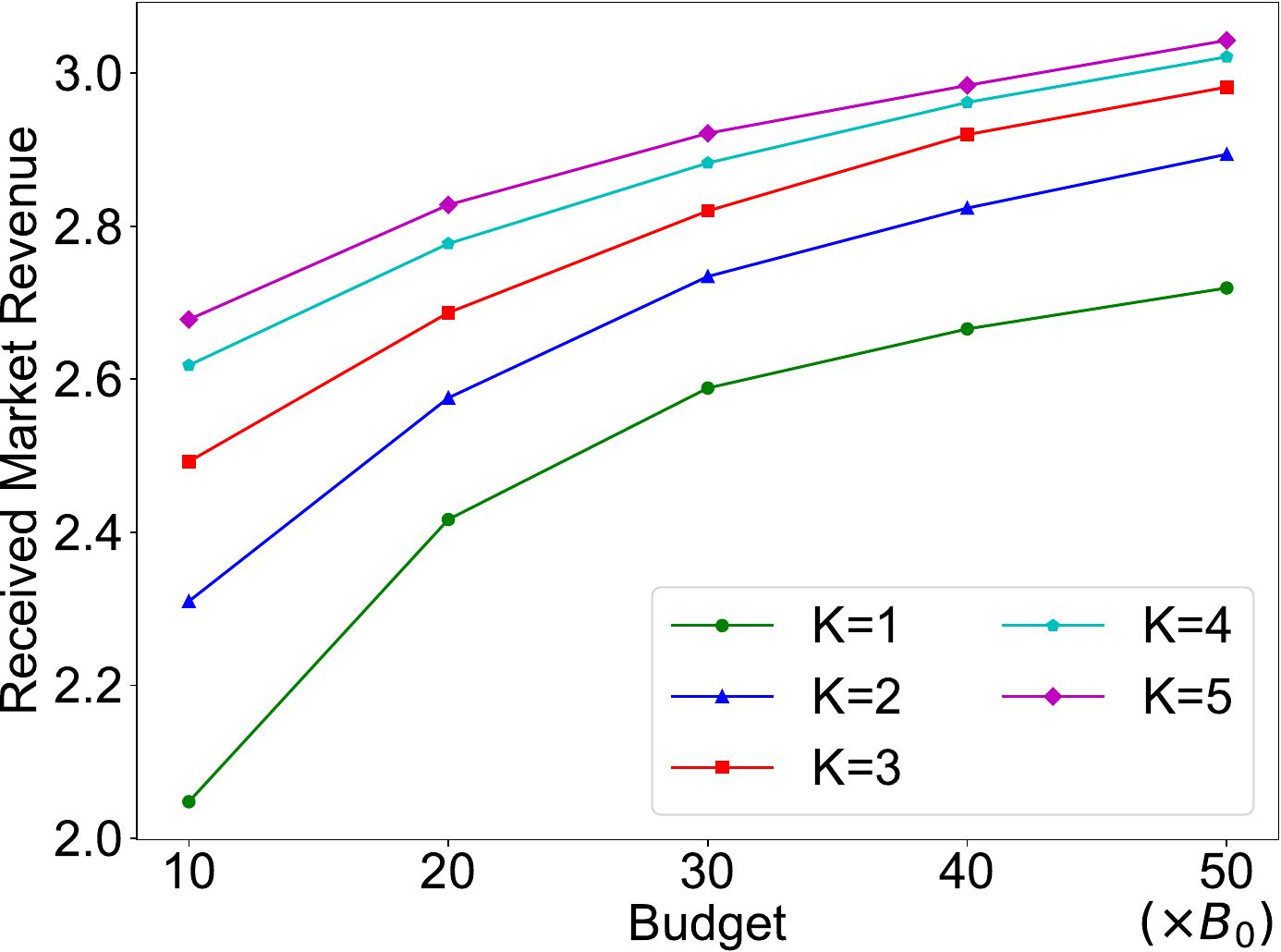}
        \caption{Apparel Dataset}
    \end{subfigure}
    \vspace{-0.1in}
    \caption{Influence of operational constraint $K$ under different $B$ on Diet and Apparel datasets.}
    \label{fig:K}
    \vspace{-0.1in}
\end{figure}

\begin{table}[!thb]
  \centering
  \caption{Computational Influence of $K$ on Diet dataset.}
  \label{tab:example_table}
  \vspace{-0.1in}
  {\resizebox{0.47\textwidth}{!}{
  \begin{tabular}{cccccc} 
    \toprule
    Value of $K$ & 1 & 2 & 3 & 4 & 5 \\
    \midrule 
   Average running time (s) & 0.1304 & 0.3214 & 0.8921 & 1.9052 & 3.5116 \\
Average reward & 2.9126 & 3.1772 & 3.3329 & 3.4148 & 3.4515 \\
   Increase of running time  & 0 & 146.1\% & 584.0\% & 1361.6\% & 2592.7\% \\
    \bottomrule 
  \end{tabular}}}
  \vspace{-0.05in}
\end{table}

\begin{figure}[!thb]
  \centering
   \begin{minipage}[b]{0.235\textwidth}
    \centering
    \includegraphics[width=\textwidth]{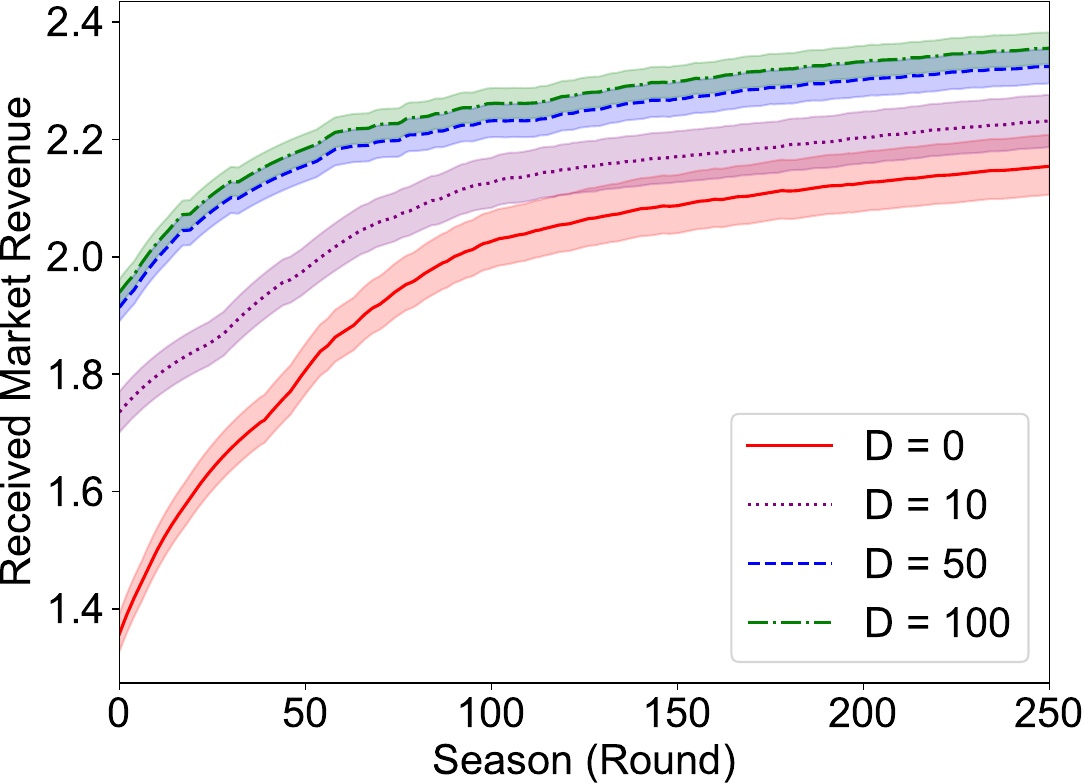}
    \vspace{-0.27in}
    \captionof{figure}{Influence of varying size $D$ on Apparel dataset.}\label{fig:histry}
  \end{minipage}
  \begin{minipage}[b]{0.235\textwidth}
    \centering
    \includegraphics[width=\textwidth]{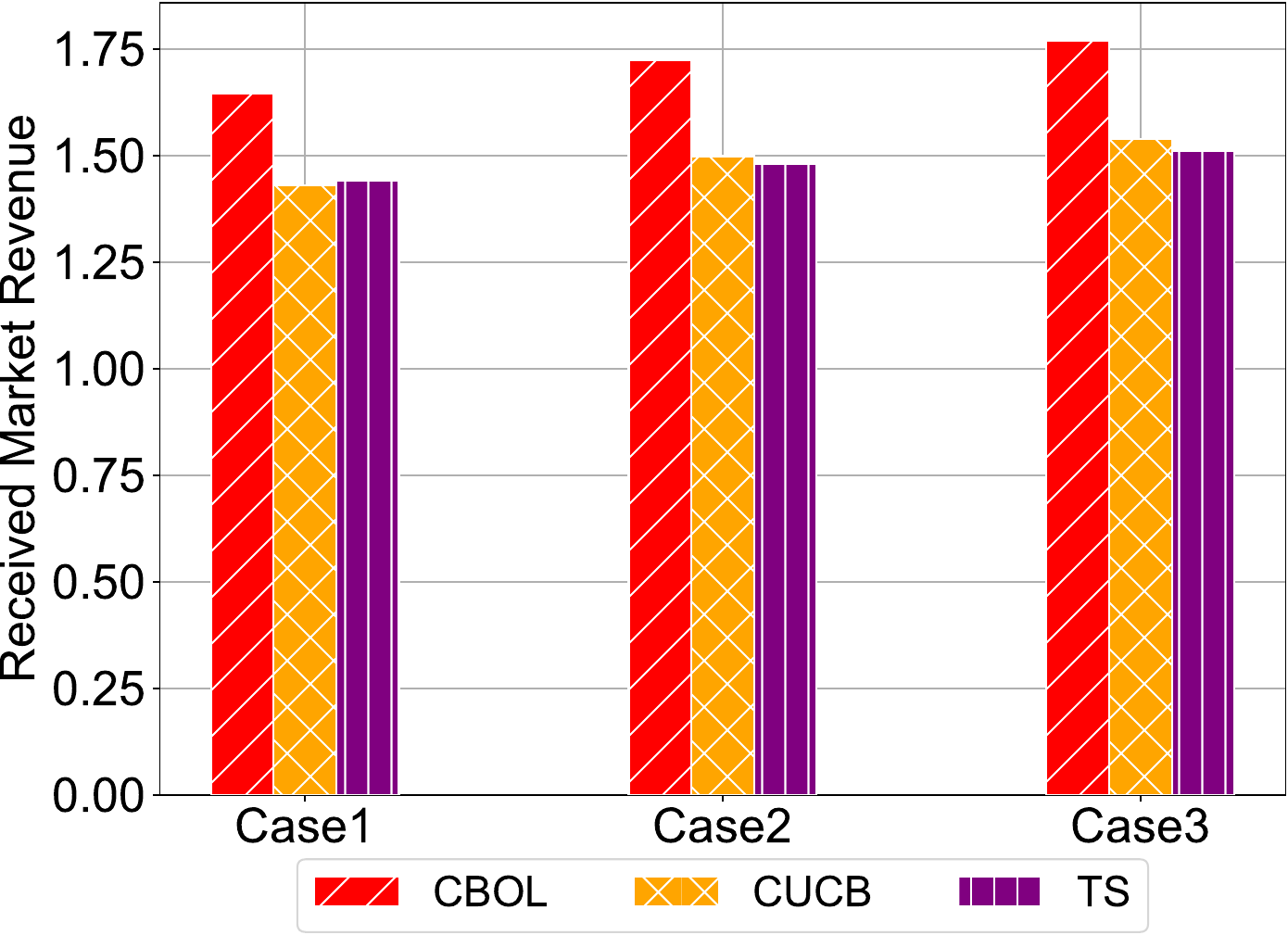}
    \vspace{-0.27in}
    \captionof{figure}{ Sensitivity analysis on IP-themed dataset.}\label{fig:case}
  \end{minipage}\hfill
  \vspace{-0.05in}
\end{figure}

\textbf{Impact of Historical Dataset (RQ1).} We test the impact of historical dataset $\cD$ of various sizes $D$ on Apparel dataset. As shown in \Cref{fig:histry}, the influence of $\cD$ is particularly significant in the early stages of online co-branding seasons compared to $D=0$. This highlights its advantage in ensuring short-term algorithmic performance, thus avoiding the performance loss typically associated with the early exploration phase in traditional MAB methods.

\textbf{Ablation Study (RQ3)}. We evaluate the performance on IP-themed dataset with different cases. Specifically, in Case 1, the budget cap \(c_u\) for each brand \(u\) is set to twice the total budget \(B\), scaled by \(u\)'s market share, and the tentative spending plan \(\cN'\) consists of two levels: \(\lfloor c_u \rfloor\) and \(\lfloor c_u/2\rfloor\), with \(T=500\). In Case 2, \(c_u\) is increased to three times \(B\), and \(\cN'\) comprises three levels: \(\lfloor c_u\rfloor\), \(\lfloor 2c_u/3\rfloor\), and \(\lfloor c_u/3\rfloor\), with \(T=1000\). In Case 3, \(c_u\) is increased to four times \(B\), and \(\cN'\) includes four levels: \(\lfloor c_u\rfloor\), \(\lfloor 3c_u/4\rfloor\), \(\lfloor c_u/2\rfloor\), and \(\lfloor c_u/4\rfloor\), with \(T=1500\). We use the two best-performing algorithms from \Cref{fig:all_reward} as baselines, and \Cref{fig:case} shows that \Cref{alg:BCUCB-T} consistently achieves the best performance, confirming its robustness.

\section{Conclusion}\label{sec:conclusion}
We presented the first unified online-offline framework for the co-branding problem, systematically tackling challenges such as resource imbalance, uncertain brand willingness, and evolving market trends. Our approach online dynamically updates a co-branding bipartite graph model to balance the exploration of new collaborations with the exploitation of proven ones, while minimizing the high initial costs often associated with co-branding. An offline optimization phase further refines sub-brand budget allocations to maximize overall returns. Through rigorous and highly non-trivial theoretical analysis, we provided regret bounds for the online learning component and approximation guarantees for the offline allocation strategy, offering interpretability and clarity regarding the underlying mechanisms. Empirical results on synthetic and real-world datasets highlight our framework’s effectiveness for guiding co-branding strategies. Future work will focus on refining the model to handle even more complex market environments.

\section*{Acknowledgement}
The work of Jinhang Zuo was supported by CityUHK 9610706.  The work of Xutong Liu was partially supported by a fellowship award from the Research Grants Council of the Hong Kong Special Administrative Region, China (CUHK PDFS2324-4S04). 
The work of John C.S. Lui was supported in part by the RGC GRF-14215722.

\balance

\clearpage
\appendix

\section{Model Extensions}\label{app:model_dis}

\subsection{Incorporating Sub-Brand Market Gains}

\begin{figure}[t]
    \centering
    \includegraphics[width=0.45\textwidth]{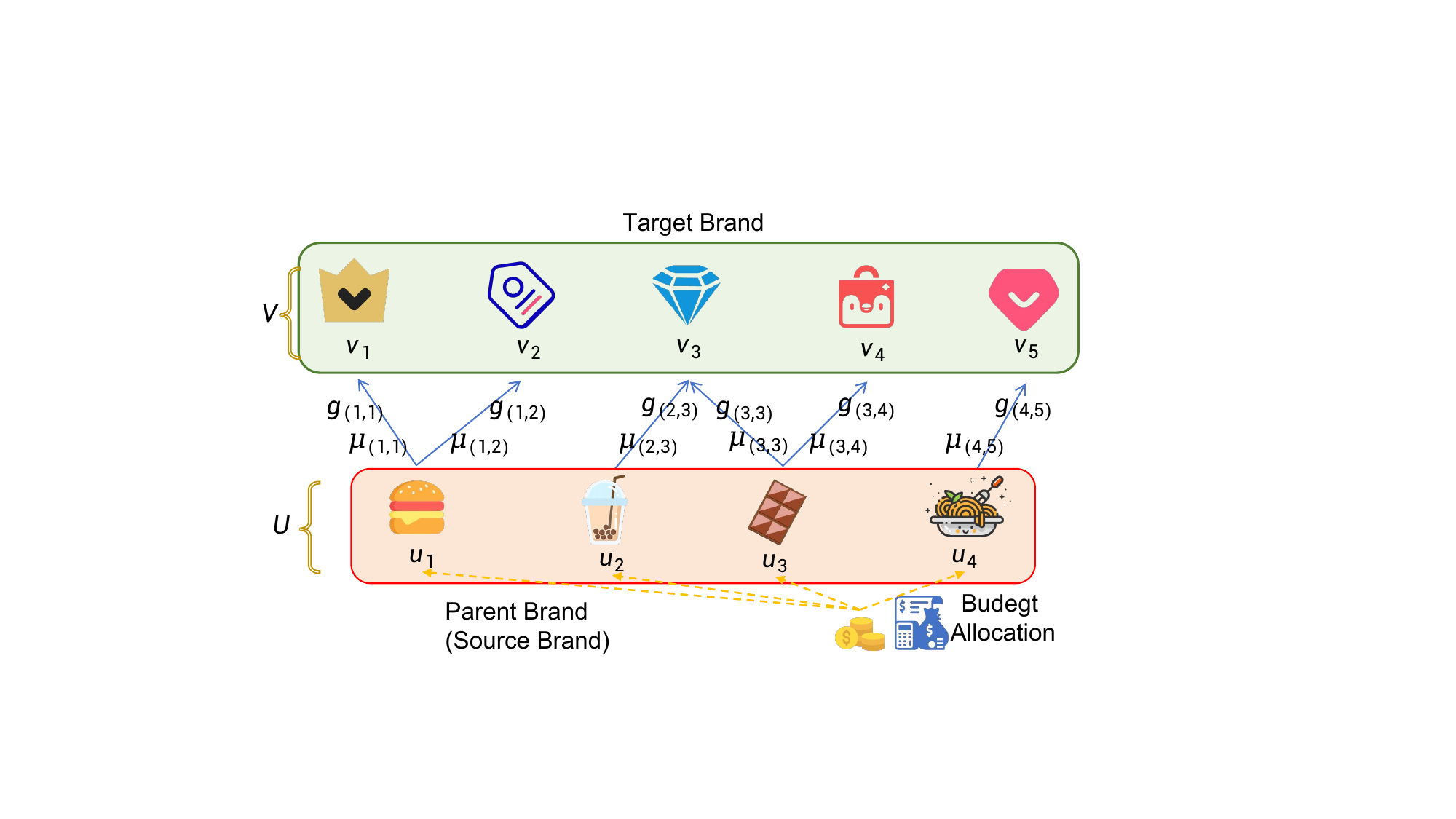}
    \caption{Extended Bipartite Graph Model for Co-Branding Incorporating Sub-Brand Market Gains.}
    \label{fig:new_model}
\end{figure}

In the main text, our model considers brand market gain as the impact of the target brand on the overall profitability of the parent company's brand portfolio, without distinguishing specific sub-brands. Here, we extend this concept to analyze the influence of the target brand on individual sub-brands. Specifically,  
as shown in \cref{fig:new_model},
the graph \( \cG \) incorporates the brand market gain vector \( \bg = \{g_e\}_{e \in \cG} \), representing the revenue earned by the parent brand from co-branding collaborations. Each element \( g_e \in \bg \) denotes the consumer market gain for partner brand pair \( e \). These gains are initially unknown and depend on the parent brand's financial investments in marketing and promotion. Additionally, for each successful co-branding pair \(e:= (u, v) \) , i.e., $X_{t,e}=1$, the  market gain  \( Y_{t,e} \) is observed based on the budget allocation $\bb$ for \( e\) during season \( t \). The expected value of \( Y_{t,e} \) is the brand market gain \( g_e \), quantifying the economic benefits of collaboration. This feedback component is observable only after successful co-branding actions.
We denote  the set of successful co-branding pair as $S'_t \in S_t$, with corresponding brand market gain feedback denoted as $\boldsymbol{Y}_{t,S'_t} = (Y_{t,1}, \ldots, Y_{t,|S'_t|})$.

Therefore, when action \( S_t \) is determined under a given budget allocation, the decision-maker receives reward  \( R_{\cG}(\bb_t) \) determined by both co-branding success probability and market gain:
$ R_{\cG}(\bb_t)=\sum_{e \in S_t}\I\{\exists\, e \in S_t \text{ s.t. } X_{t, e}=1\}Y_{t,e}.$
Given the definitions of \( \bmu \) and \( \bg \), the expected reward  $r_{\cG}(\bb)$  is given by: $r_{\cG}(\bb) = \sum_{e \in S} g_e \mu_{e}.$
This formulation reveals a linear summation. Consequently, our algorithm remains applicable to this simpler model and its formulation. Specifically, the elements of $\cA'$ are modified to include
$e \in \cA'$.
Lines \ref{line: extend over online ucb for nodes} and \ref{line: update weight} in Algorithm \ref{alg:BCUCB-T} are updated as follows: $\rho'_{e} \leftarrow \sqrt{3\ln t/(2T'_{t,e})}$, $\hat{g}_{e} \leftarrow \hat{g}_{e} + \rho'_{e}$ and 
 $T'_{e} \leftarrow T'_{t,e} + 1$,  $\hat{g}_e \leftarrow \hat{g}_e  + (Y_e - \hat{g}_e )/T'_{t,e}$, 
	   $\hat{V}'_{e}\leftarrow \frac{T_{t,e}'-1}{T_{t,e}'}\left(\hat{V}_{e}' + \frac{1}{T_{t,e}'} \left(\hat{g}_{e} - Y_{t,e}\right)^2\right)$, respectively.
These updates allow the algorithm to incorporate brand market gains effectively in this extended scenario.

\subsection{Target Brand Exclusivity}

We consider scenarios where multiple sub-brands \( u \) under the same parent company \( \mathcal{U} \) can collaborate with any target brand \( v \in \mathcal{V} \) without exclusivity agreements in the main text. This implies that a target brand \( v \) is not restricted to partnering with a single specific sub-brand from \( \mathcal{U} \). For example,
target brand \( v_1 \) can collaborate with both \( u_1 \) and \( u_3 \), without violating exclusivity agreements.
In contrast, we also consider a simplified yet commonly encountered case where each target brand is limited to partnering with a single sub-brand from a given parent company.  In this context, the reward function can be decomposed into separate components on each source sub-brand, and the expected reward can be expressed as: $ r_{\cG} (\bb)=\sum_{u\in\cU}\sum_{v \in \cV_u}g_u \mu_{u,v,b_u},$
where \( \mathcal{V}_u \) represents the set of target brands associated with the \( u \)-th sub-brand.
In the absence of exclusivity agreements, multiple base arms (i.e., co-branding pairs) may share resources or be influenced by the same target brand's market performance. This overlap introduces complexities in coordination and estimation, as the interactions between different co-branding pairs become interdependent.
Conversely, when a target brand can collaborate with only one sub-brand from a parent brand, the marginal gains of each base arm become independent. This independence eliminates the need for complex joint distribution modeling, thereby simplifying the budget allocation process. 
Under this exclusivity constraint, classical dynamic programming techniques \cite{goda2023dynamic,nemhauser1969discrete} can be employed to achieve the optimal offline budget allocation, ensuring an approximation ratio of \( \alpha = 1 \).

\subsection{Modeling Budget Effects on Market Gains}
In \Cref{sec:model}, we model the brand market gain vector \( \bg = \{g_v\}_{v \in \cV} \) and primarily considered that the market gain is mainly driven by consumer preferences rather than budget allocation. This is because in a complex market environment, a higher investment does not necessarily imply higher returns. Therefore, in our original model, we did not directly associate \( \bg \) with the budget allocation vector \( \bb \).
In fact, we can also construct a market gain model that considers the impact of budget allocation. Similar to the handling of the co-branding probability \( \mu_{(u,v),b_u} \), for any \( v \in \cV \) we can extend \( g_v \) to a form that depends on the budget allocation, denoted as \( g_{v,b_u} \).

Under this extension, the complexity of the model is significantly increased, and the arm space \( \cA' \) required for online learning becomes correspondingly larger. However, our algorithm can still handle this situation effectively. Specifically, in at line \ref{line: extend over online ucb for nodes} of \Cref{alg:BCUCB-T}, we need to modify the update rule as follows. For each \( (v, s) \in \cA' \), update the confidence radius as follows $\rho'_{v,s}$:
\begin{equation}
    \rho'_{v,s} \leftarrow \sqrt{\frac{6\hat{V}_{v,s} \log t}{T_{t,v,s}}} + \frac{9 \log t}{T_{t,v,s}},
\end{equation}
and then update the UCB-based estimate of the market gain:
\begin{equation}
    \hat{g}_{v,s} \leftarrow \hat{g}_{v,s} + \rho'_{v,s}.
\end{equation}
This way, not only is the model structurally more flexible, but the algorithm is also capable of effectively performing online learning in a larger exploration space.

\section{Proof Appendix}\label{app: proof}
We begin by introducing some definitions. 
Recall that $\cA \cup \cA'$ represents the set of all base arms. Without loss of generality, assume $X_{t,e} \in [0,1], Y_{t,v} \in [0,1]$ for any $t$.  Following the notation in \cite{wang2017improving} and with slight abuse of notation, we use $X_a \in [0,1]$ to consistently represent both $X_{e}$ and $Y_v$ as 
the random outcome of base arm $a \in \cA\cup\cA'$. Define $D$ as the unknown joint distribution of random outcomes $\bm{X}=(X_a)_{a\in \cA \cup \cA'}$ with unknown mean $\mu_a=\E_{\bm{X}\sim D}[X_a]$. The probability that a budget allocation $\bb$ triggers arm $a$ (i.e. $X_a$ is observed) under the unknown environment distribution $D$ is denoted as $p_a^{ \bb}$ the probability that budget allocation. For simplicity, we omit the dependency on $D$ when there is no ambiguity.
Specifically, $p_a^{\bb}=1$ for $a \in \cA$, $p_a^{\bb}=1-\prod_{e\in \cE}(1-\mu_{e,a,s_u})$ for $a \in \cA'$, and $p_a^{\bb}=0$ for the rest of base arms.
We define the reward gap $\Delta_{\bb}$ as $\max(0, \alpha r_{\cG}(\bb^*) - r_{\cG}(\bb))$ for all feasible budget allocation $\bb$. For each base arm $a$, 
we further define  $\Delta_{\min}^a{=}\min_{\Delta_{\bb}>0, b_u = s}\Delta_{\bb}$ and $\Delta_{\max}^a{=}\max_{\Delta_{\bb}>0, b_u=s}\Delta_{\bb}$. 
If there is no budget allocation $\bb$ with $b_u=s$ such that $\Delta_{\bb} {>} 0$, we adopt the convention: $\Delta_{\min}^a{=}\infty$ and $\Delta_{\max}^a{=}0$. 
Let $\Delta_{\min} {=} \min_{a \in \cA \cup \cA'}\Delta_{\min}^a$ and  $\Delta_{\max} {=} \max_{a \in \cA \cup \cA'}\Delta_{\max}^a$.
Define the set of observed base arms at $t$ as $\tau_t$, which includes both co-branding pairs with allocated budgets $\{(e,s)\in \cA| b_u=s\} $ and the corresponding arms in $\cA'$.

\subsection{Instance-Dependent Regret Analysis}\label{subapp:dependent}

\begin{lemma} \label{thm: over regret}
Algorithm~\ref{alg:BCUCB-T} has the following instance-dependent  regret bound   $O\bigg((NUV+V)$ $ \bigg(\frac{ \log (UV+V) }{\Delta^{a}_{\min}} $ $+ \log (\frac{UV+V
}{\Delta^{a}_{\min}})\bigg)\log T\bigg)$.
\end{lemma}

\begin{proof}
Recall that \(S = \{(e,s) \in \cA \mid s = b_{u},s\in \cN_u\}\) represents the set of co-branding pairings under a specific budget allocation. Let \(\tilde{S}\) denote the set of base arms that can be triggered by \(S\), i.e., \(\{a \in \cA \cup \cA' \mid p_a^{ \bb} > 0\}\). 
Using the Lemma 9 from \cite{liu2022batch},
we decompose the regret $Reg(T)$ into two event-filtered components, $Reg(T, E_{t,1})$ and $Reg(T,E_{t,2})$:
$Reg(T) \le Reg(T, E_{t,1})+Reg(T, E_{t,2}),$
where event $E_{t,1}=\{\Delta_{\bb_t} \le 2e_{t,1}(\bb_t)\}$, event $E_{t,2}=\{\Delta_{\bb_t} \le 2e_{t,2}(\bb_t)\}$, with error terms defined as $e_{t,1}(\bb_t)=4\sqrt{3.75V}p_{a}^{\bb_t}\sqrt{\sum_{a\in \tilde{S}_t}(\frac{\log t}{T_{t-1,a}}\wedge \frac{1}{28})}$, $e_{t,2}(\bb_t)=28\sum_{a\in \tilde{S}_t}p_{a}^{\bb_t}(\frac{\log t}{T_{t-1,a}}\wedge \frac{1}{28})$.   Since the confidence radius for each base arm does not incorporate feedback from the online market, we treat all past selections of a base arm as a single pull. The feedback for this single selection can be reviewed as the average reward over all historical pulls. Consequently, regardless of the distribution of the finite historical data $\cD$, the regret from historical arm selections is bounded by a time-independent constant.

For $Reg(T, E_{t,1})$, denote $\tilde{\Delta}_{t,\bb_t}=\frac{\Delta_{t,\bb_t}}{p_a^{\bb_t}}$ for budget allocation $\bb_t$ at $t$, and $\tilde{\Delta}_{a}^{\min}=\min_{\bb : p_a^{\bb}>0, \Delta_{t,\bb}>0}\Delta_{t,\bb}/p_a^{\bb}$. Regarding $ \Delta_{\bb_t}$ on the selected budget allocation $\bb_t$ at  $t$, we have: $\Delta_{\bb_t} \le \sum_{a\in \cA \bigcup \cA'} $ $ \frac{4(4\sqrt{3})^2 V (p_{a}^{\bb_t})^2\frac{\log t}{T_{t-1,a}}}{\Delta_{\bb_t}}$ $\le \sum_{a\in \cA \bigcup \cA'}$ $ p_a^{\bb_t} \left(\frac{384 V \frac{\log t}{T_{t-1,a}}}{\tilde{\Delta}_{\bb_t}}-\frac{\tilde{\Delta}_{\bb_t}}{UV+V}\right).
$
Using a modified version of Lemma 10 in \cite{liu2022batch}, we define a regret allocation function as follows:
$
\kappa_{a,T}(\ell) = 
\frac{48 V} {\tilde{\Delta}_{a}^{\min}} $ $ \text{for $\ell=0$;}$ $
2\sqrt{\frac{192 V \log T}{\ell }}$ $ \text{for $1\le\ell\le \frac{192 V\log T}{(\tilde{\Delta}_{a}^{\min})^2}$;}$ $
\frac{384 V \log T}{\ell\tilde{\Delta}_{a}^{\min} }$ $ \text{for $\frac{192 V\log T}{(\tilde{\Delta}_{a}^{\min})^2} < \ell \le \frac{384 V^2(U+1)\log T}{(\tilde{\Delta}_{a}^{\min})^2}$;}$ $
0 $ $ \text{otherwise.}
$
Using $  \Delta_{\bb_t} =\E_t[\Delta_{\bb_t}] $,  we  obtain:
$\Delta_{\bb_t} =\E_t[\Delta_{\bb_t}]   $ $\le  \E_t\bigg[\sum_{a\in \cA \bigcup \cA'}(\frac{ 384  V \frac{\log t}{T_{t-1,a}}}{\tilde{\Delta}_{\bb_t}}-$ $\frac{\tilde{\Delta}_{\bb_t}}{UV+V})p_a^{\bb_t}\bigg]$ $\le \E_t\left[\sum_{a \in \tau_t}{ \kappa_{a,T}(T_{t-1,a})}\right].
$
By the tower rule and the fact that $T_{t-1,a}$ is increased by $1$ if and only if $a \in \tau_t$, we can derive:  
$ \E\left[\sum_{t\in [T]}\E_t\left[\sum_{a \in \tau_t}{ \kappa_{a,T}(T_{t-1,a})}\right]\right] $ $=\E\left[\sum_{a\in \cA \bigcup \cA'}\sum_{j=0}^{T_{T-1,a}}\kappa_{a,T}(j)\right].$
 Applying  the definition of $\kappa_{a,T}(\ell)$, the expected regret  
$ Reg(T, E_{t,1})= \E\left[\sum_{t=1}^T\Delta_{\bb_t}\right]$ under event $E_{t,1}$ is bounded as:
$
Reg(T, E_{t,1}) \le$ $ \sum_{a\in \cA \bigcup \cA'}\bigg(\frac{48  V }{\tDelta_{a}^{\min}}+ $ $ \frac{384  V \log T}{\tDelta_{a}^{\min} } (1+\log (UV+V)) +$ $   \frac{768  V \log T}{\tDelta_{a}^{\min} } \bigg).    
$Here, $|\cA|= NUV, |\cA'|=V$. 

In the analysis of $Reg(T,E_{t,2})$ under event $E_{t,2}$, similar to that of $Reg(T,E_{t,1})$, we can derive:
$ \Delta_{\bb_t} \le \sum_{a\in \tilde{S}_t} 56 p_{a}^{\bb_t} \min\left\{1/28, \frac{\log T}{T_{t-1,a}}\right\} $ $ \le -\Delta_{\bb_t} + 2\sum_{a\in \tilde{S}_t} 56 p_{a}^{\bb_t} \min\left\{1/28, \frac{\log T}{T_{t-1,a}}\right\}$ $ \le \sum_{a\in \cA \bigcup \cA'} $ $p_a^{\bb_t} $ $\bigg(-\frac{\Delta_{\bb_t}}{UV+V}$ $+ 112  \min\left\{1/28, \frac{\log T}{T_{t-1,a}}\right\}\bigg).
$ Different from \cite{liu2022batch} that using $p_a^{\bb_t}=\E_t[\I\{a \in \tau_t\}]$, we directly establish the probability equivalence: $
\E_t\bigg[\sum_{a\in \cA \bigcup \cA'}p_a^{\bb_t}$ $ \left(-\frac{\Delta_{\bb_t}}{UV+V}+ 112  \min\left\{\frac{1}{28}, \frac{\log T}{T_{t-1,a}}\right\}\right)\bigg] 
= $ $\E_t\bigg[ \sum_{a \in \tau_t}$  $ \bigg(-\frac{\Delta_{\bb_t}}{UV+V}+ $ $ 112  \min\left\{\frac{1}{28}, \frac{\log T}{T_{t-1,a}}\right\}\bigg)\bigg]
$, which minimizes the observation randomness. It then follows that: $ \Delta_{\bb_t}   \le  \E_t\left[\sum_{a \in \tau_t}{ \kappa_{a,T}(T_{t-1,a})}\right],
$
where the regret allocation function $\kappa_{a,T}(\ell)$ is adjusted as:
$\kappa_{a,T}(\ell) = 
\Delta_{a}^{\max} \text{ if } 0 \leq \ell \leq \frac{112 \log T}{\Delta_{a}^{\max}};\; 
\frac{112 \log T}{\ell} \text{ if } \frac{112 \log T}{\Delta_{a}^{\max}} < \ell \leq \frac{112 (UV+V) \log T}{\Delta_a^{\min}};\;$ $ 
0 \text{ if } \ell > \frac{112 (UV+V) \log T}{\Delta_a^{\min}} + 1.$
Consequently, the expected regret $Reg(T, E_{t,2})$ under event $E_{t,2}$ can be bounded as:
$
    Reg(T, E_{t,2})
    \le \E\left[\sum_{t\in [T]}\E_t\left[\sum_{a \in \tau_t}{ \kappa_{a,T}(T_{t-1,a})}\right]\right]$ $=\E\left[\sum_{a\in \cA \cup \cA'}\sum_{s=0}^{T_{t-1,a}}\kappa_{a,T}(s)\right]$ $\le  \sum_{a\in \cA \bigcup \cA'} 112  \left(1+\log \left(\frac{(UV+V)\Delta_a^{\max}}{\Delta_{a}^{\min}}\right)\right)\log T$ $+UV \Delta_{\max}.
$
\end{proof}

\subsection{Proof of Theorem \ref{thm: over regret}}
Building on the instance-dependent result in Lemma \ref{thm: regret}, we now present the proof of the instance-independent regret bound in Theorem \ref{thm: over regret} for \Cref{alg:BCUCB-T}.
\begin{proof}
Since the minimum gap \( \Delta^{a}_{\min} \) can be extremely small in some instances, we derive regret bounds for \Cref{alg:BCUCB-T} that are independent of specific co-branding problem instances and do not directly rely on \( \Delta_{t,\bb} \). To achieve this, we introduce a fixed gap parameter \( \Delta \), the value of which will be specified later. Based on Lemma \ref{thm: regret}, We analyze the regret under two scenarios: \( \{\Delta_{t,\bb_t} \le \Delta\} \) and \( \{\Delta_{t,\bb_t} > \Delta\} \).
For the first scenario where \( \Delta_{t,\bb_t} \le \Delta \), the regret is trivially upper bounded by:
$Reg(T, \{\Delta_{t,\bb_t} \le \Delta\}) \le T\Delta $.
For the second scenario where \( \Delta_{t,\bb_t} > \Delta \),  we simplify the regret analysis by approximating all minimal gaps \( \Delta_a^{\min} \) with \( \Delta \). This results in the following bound:
$ Reg(T, \{\Delta_{t,\bb_t} > \Delta\}) \le O(
(NUV+V)\log T(
\frac{V}{\Delta} \log (UV+V) + \log (\frac{UV+V}{\Delta})).
$
To balance these two scenarios, we carefully choose \( \Delta \) as:
$ \Delta = \Theta (\sqrt{\frac{(NU+1)V^2\log T \log (UV+V)}{T}} + \frac{(NUV+V)\log (UV+V)\log T}{T})$. By substituting this choice of $\Delta$, we derive a new regret bound for \Cref{alg:BCUCB-T} that is independent of the variations in gaps:
$
    Reg(T) \le O ( V\sqrt{(NU+1)(\log (U+1)V)  T\log T} + \log((U+1)VT)\log T ). 
$
This formulation ensures that the regret bound does not depend on specific gap values, making it more robust to variations in different co-branding instances.
\end{proof}

\begin{figure*}[!htb]
    \centering
    \begin{subfigure}[b]{0.283\linewidth}
        \includegraphics[width=\linewidth]{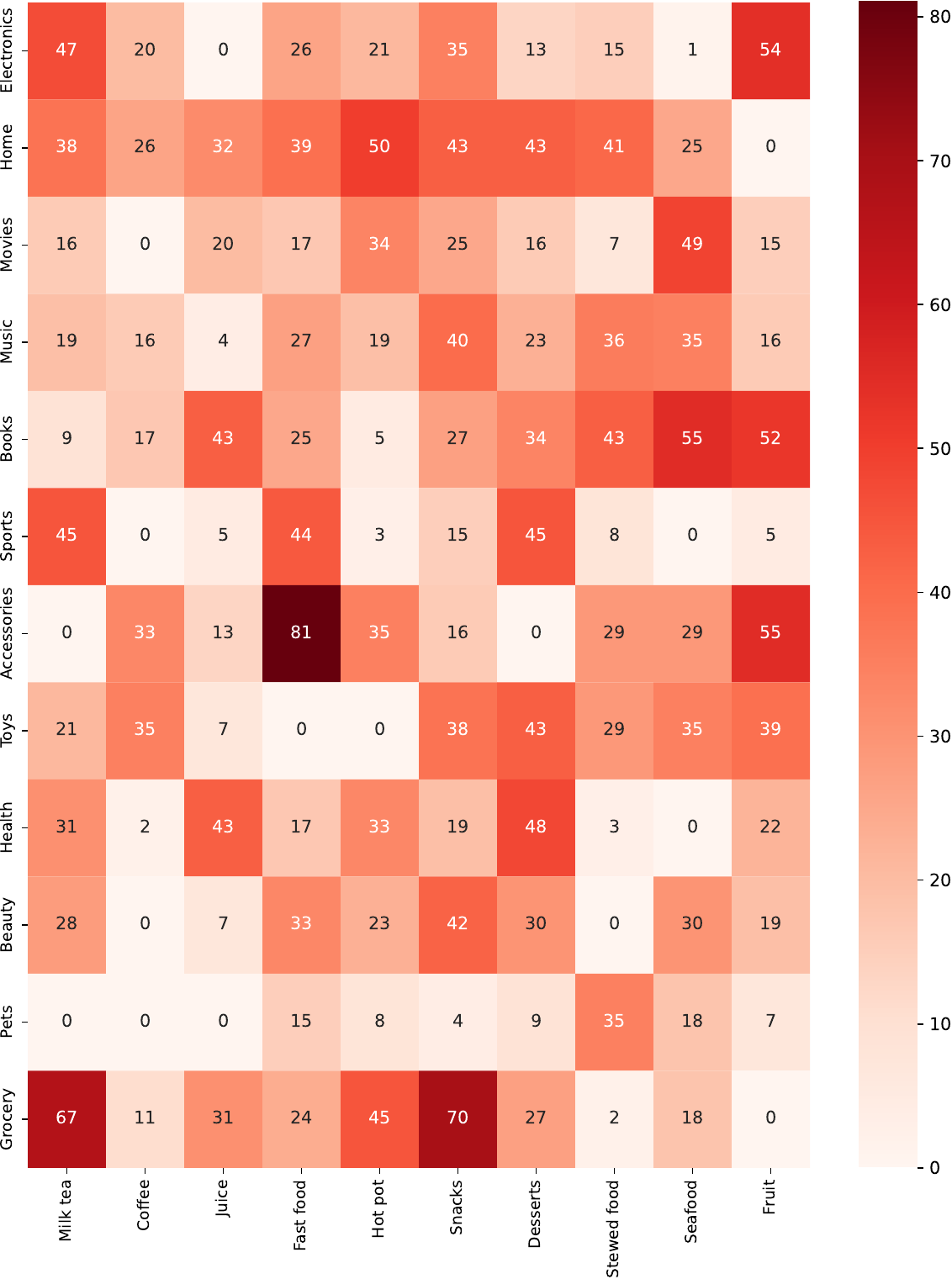}
        \caption{Diet Dataset} 
    \end{subfigure}
    \begin{subfigure}[b]{0.38\linewidth}
        \includegraphics[width=\linewidth]{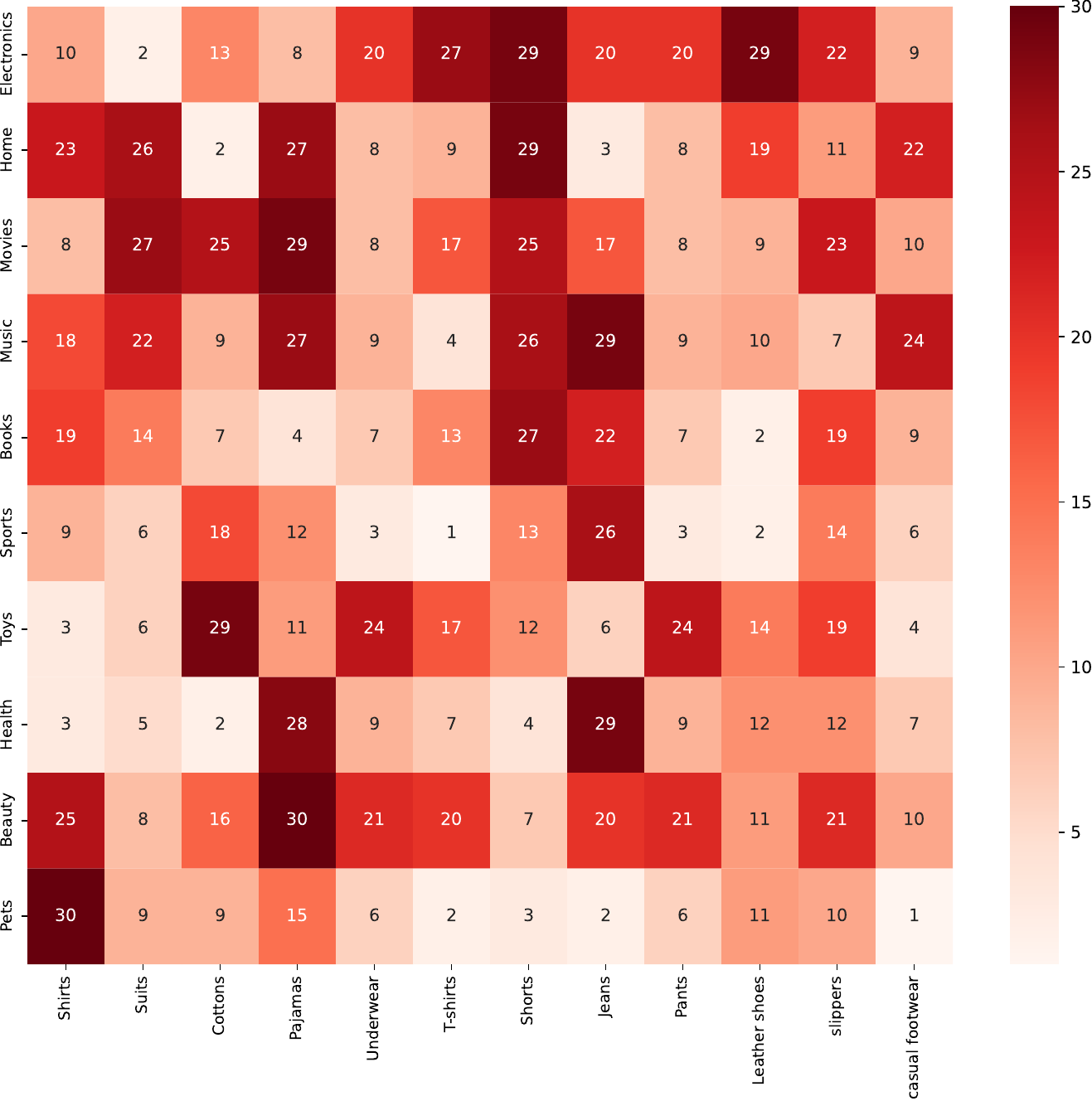}
        \caption{Apparel Dataset}
    \end{subfigure}
    \begin{subfigure}[b]{0.315\linewidth}
        \includegraphics[width=\linewidth]{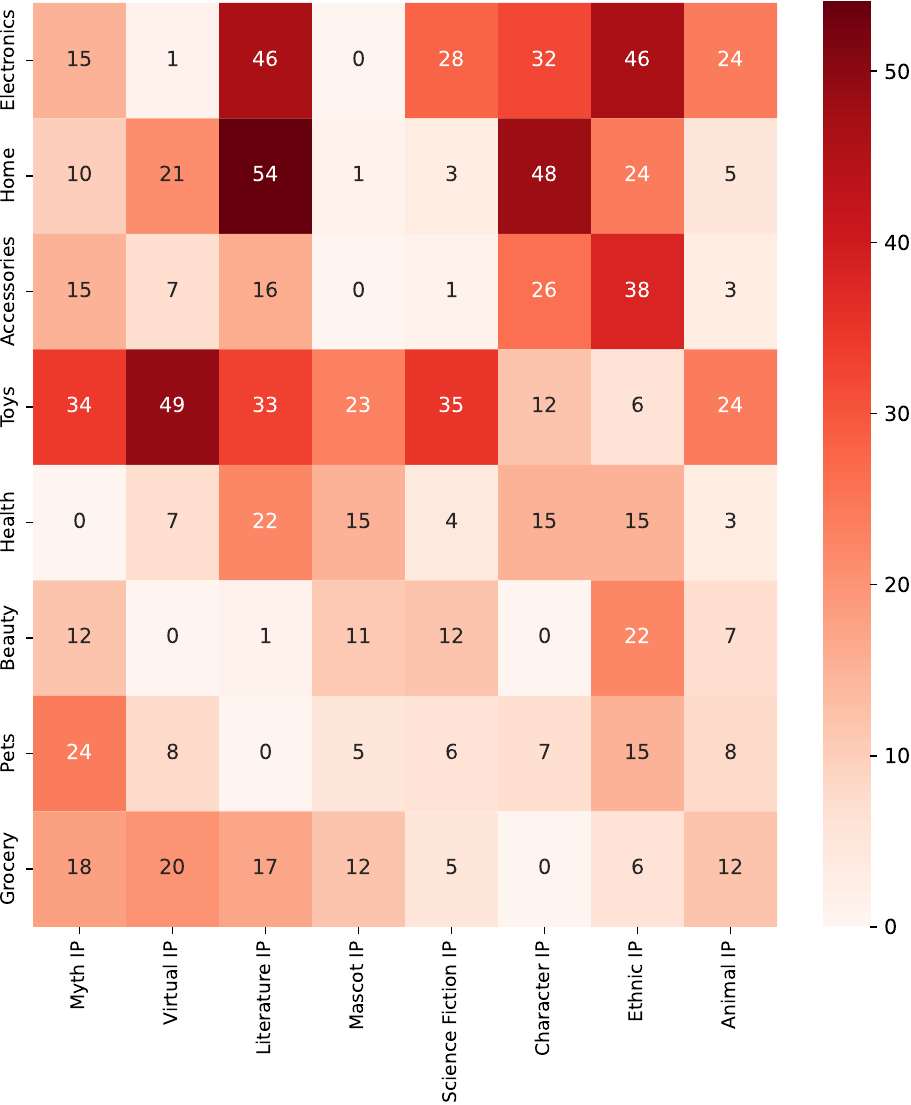}
        \caption{IP-themed Dataset} 
    \end{subfigure}
    \vspace{-0.1in}
    \caption{Visualization of co-branding relationships across Diet, Apparel, and IP-themed brands.}
    \label{fig:visu}
\end{figure*}

\subsection{Proof of Theorem \ref{thm:offline}}\label{subapp:offline}

\begin{proof} 
Consider a starting value given by a partial solution $\bb^1 \in \mathbb{N}_0^U$.
We proceed by reordering the optimal solution $\bb^*$ based on a non-increasing ordering of marginal gains. Specifically, the pair $(i^*_1, s^*_1)$ is chosen such that its marginal gain concerning the empty set is the highest among all possible pairs. Similarly, the pair $(i^*_2, s^*_2)$ is selected such that its marginal gain, evaluated concerning the solution containing only the pair $(i^*_1, s^*_1)$, is the highest among the remaining pairs, and so forth.
In formal terms, the optimal solution can be expressed as $\bb^* = \{(i^*_1, s^*_1), ..., (i^*_U, s^*_U)\}$, where each index $i^*_u \in \cU \setminus \{i^*_1, ...,i^*_{u-1}\}$ is selected to maximize the following expression:
$
\left(r_{\cG}\left(\sum_{j=1}^{u-1}s^*_j\bch_{i^*_j} + {s^*}_u\bch_{{i^*_u}} \right) -  r_{\cG}\left(\sum_{j=1}^{u-1}s^*_j\bch_{i^*_j}\right)\right).
$
Next, we aim to bound the ``virtual'' marginal gain $\Delta^i = r_{\cG}(\bb^{i+1}) - r_{\cG}(\bb^i)$. Recall that $i$ represents the first iteration where the current solution cannot be extended.
Without loss of generality, we assume the virtual pair $(u^i, s^i)$ contributes to improving $\bb^i$ towards the optimal budget allocation, i.e., $s^i+b^i_{u^i} \le b^*_{u^i}$. 
Otherwise, if $s^i+b^i_{u^i} > b^*_{u^i}$, we can safely remove $(u^i, s^i)$ from the queue $\cQ$ without affecting the greedy solution, the optimal solution, or the analysis.

Starting from the initial solution $\bb^1$ such that $\bb^1$ matches $(i^*_1, i^*_1),$ $ (i^*_2, s^*_2)$, and $ (i^*_3, s^*_3)$, i.e., ${\bb^1}_{i^*_u} = s^*_i$ for $i=1,2,3$, we aim to bound $\Delta^i = r_{\cG}(\bb^{i+1}) - r_{\cG}(\bb^i) \le r_{\cG}(\bb^{i} \vee b^*_{u^i}\bch_{u^i}) - r_{\cG}(\bb^i) \le r_{\cG}(b^*_{u^i}\bch_{u^i}) - r_{\cG}(\bm{0}) \le r_{\cG}({s^*_1}\bch_{i^*_1})$.
Using marginal gain properties, we have: $\Delta^i = r_{\cG}(\bb^{i+1}) - r_{\cG}(\bb^i) \le r_{\cG}(\bb^{i} \vee b^*_{u^i}\bch_{u^i}) - r_{\cG}(\bb^i) = r_{\cG}(\bb^{i} \vee s^*_1\bch_{i^*_1} \vee b^*_{u^i}\bch_{u^i}) - r_{\cG}(\bb^i \vee s^*_1\bch_{i^*_1}) \le r_{\cG}( s^*_1\bch_{i^*_1} \vee b^*_{u^i}\bch_{u^i}) - r_{\cG}( s^*_1\bch_{i^*_1}) \le r_{\cG}( s^*_1\bch_{i^*_1} +  s^*_2\bch_{i^*_2}) - r_{\cG}( s^*_1\bch_{i^*_1}).$
By repeating the above reasoning, we also have $\Delta^i \le r_{\cG}( s^*_1\bch_{i^*_1} +  s^*_2\bch_{i^*_2} +  s^*_3\bch_{i^*_3}) - r_{\cG}( s^*_1\bch_{i^*_1} +  s^*_2\bch_{i^*_2}).$
Adding these inequalities together yields: $\Delta^i \le r_{\cG}(\bb^1)/3$.
Now, applying Lemma 9 from \cite{liu2021multi} with $l = i$ and noting that $\sum_{j=1}^u{s^j} \ge B$, we establish: $r_{\cG}(\bb^{i+1}) \ge r_{\cG}(\bb^1) + (1-1/e)(r_{\cG}(
\bb^*) - r_{\cG}(\bb^1))$.
Combining this with $\Delta^i \le r_{\cG}(\bb^1)/3$, we derive: $r_{\cG}(\bb^{i}) \ge (1-1/3)r_{\cG}(\bb^1) + (1-1/e)(r_{\cG}(
\bb^*) - r_{\cG}(\bb^1)) \ge (1-1/e)r_{\cG}(\bs^*)$. Note that if $K < 3$ (i.e., $1/K > 1/3$), the inequality does not hold, and the $1 - 1/e$ approximation ratio cannot be achieved. Therefore, we set $K = 3$ to ensure the approximate ratio while minimizing the time complexity as much as possible.
\end{proof}

\subsection{Proof of Lemma~\ref{lem: lattice submodular}}\label{subapp:proof_lem}
\begin{proof}

\textbf{``If'' Direction.}
We assume that \( f(\boldsymbol{x} + \boldsymbol{\chi}_i) - f(\boldsymbol{x}) \geq f(\boldsymbol{x} + \boldsymbol{\chi}_j + \boldsymbol{\chi}_i) - f(\boldsymbol{x} + \boldsymbol{\chi}_j) \) holds for all \(\boldsymbol{x} \in \mathbb{N}_0^U\) and distinct \(i, j \in \cU\). Our goal is to prove that \( f(\boldsymbol{x} + \boldsymbol{\chi}_i) - f(\boldsymbol{x}) \geq f(\boldsymbol{y} + \boldsymbol{\chi}_i) - f(\boldsymbol{y}) \) for any \(\boldsymbol{x} \leq \boldsymbol{y}\) (i.e., \(x_i \leq y_i\) for all \(i \in \cU\)) and \(i \in \cU\) such that \(x_i = y_i\).

Define \( I_0 = \{ i \in \cU : x_i = y_i \} \) and \( I_1 = \{ i \in \cU : x_i < y_i \} \). Since \(\boldsymbol{x} \leq \boldsymbol{y}\), we can write \(\boldsymbol{y} = \boldsymbol{x} + \sum_{j=1}^s \alpha_j \boldsymbol{\chi}_{i_j}\), where \(I_1 = \{ i_1, \dots, i_s \}\) and \(\alpha_j = y_{i_j} - x_{i_j} \geq 1\). For any \(i \in I_0\), we start with \( f(\boldsymbol{x} + \boldsymbol{\chi}_i) - f(\boldsymbol{x}) \). Using the assumption with \(j = i_1\), we have \( f(\boldsymbol{x} + \boldsymbol{\chi}_i) - f(\boldsymbol{x}) \geq f(\boldsymbol{x} + \boldsymbol{\chi}_{i_1} + \boldsymbol{\chi}_i) - f(\boldsymbol{x} + \boldsymbol{\chi}_{i_1}) \). Applying the assumption again to \(\boldsymbol{x} + \boldsymbol{\chi}_{i_1}\), we get \( f(\boldsymbol{x} + \boldsymbol{\chi}_{i_1} + \boldsymbol{\chi}_i) - f(\boldsymbol{x} + \boldsymbol{\chi}_{i_1}) \geq f(\boldsymbol{x} + 2\boldsymbol{\chi}_{i_1} + \boldsymbol{\chi}_i) - f(\boldsymbol{x} + 2\boldsymbol{\chi}_{i_1}) \). We continue this process incrementally, adding \(\alpha_1 \boldsymbol{\chi}_{i_1}\), then \(\alpha_2 \boldsymbol{\chi}_{i_2}\), and so on, up to \(\alpha_s \boldsymbol{\chi}_{i_s}\), arriving at \( f(\boldsymbol{x} + \sum_{j=1}^s \alpha_j \boldsymbol{\chi}_{i_j} + \boldsymbol{\chi}_i) - f(\boldsymbol{x} + \sum_{j=1}^s \alpha_j \boldsymbol{\chi}_{i_j}) = f(\boldsymbol{y} + \boldsymbol{\chi}_i) - f(\boldsymbol{y}) \). Thus, the chain of inequalities gives \( f(\boldsymbol{x} + \boldsymbol{\chi}_i) - f(\boldsymbol{x}) \geq f(\boldsymbol{y} + \boldsymbol{\chi}_i) - f(\boldsymbol{y}) \).
Next, for \(i \in I_0\) and any \(a \in \mathbb{N}_0^U\), we express \( f(\boldsymbol{x} + a \boldsymbol{\chi}_i) - f(\boldsymbol{x}) = \sum_{k=1}^a \left[ f(\boldsymbol{x} + k \boldsymbol{\chi}_i) - f(\boldsymbol{x} + (k-1) \boldsymbol{\chi}_i) \right] \) and similarly \( f(\boldsymbol{y} + a \boldsymbol{\chi}_i) - f(\boldsymbol{y}) = \sum_{k=1}^a \left[ f(\boldsymbol{y} + k \boldsymbol{\chi}_i) - f(\boldsymbol{y} + (k-1) \boldsymbol{\chi}_i) \right] \). For each \(k\), since \(\boldsymbol{x} + (k-1) \boldsymbol{\chi}_i \leq \boldsymbol{y} + (k-1) \boldsymbol{\chi}_i\) with equal \(i\)-th coordinates, the previous result implies \( f(\boldsymbol{x} + k \boldsymbol{\chi}_i) - f(\boldsymbol{x} + (k-1) \boldsymbol{\chi}_i) \geq f(\boldsymbol{y} + k \boldsymbol{\chi}_i) - f(\boldsymbol{y} + (k-1) \boldsymbol{\chi}_i) \). Summing over \(k\) from 1 to \(a\), we obtain \( f(\boldsymbol{x} + a \boldsymbol{\chi}_i) - f(\boldsymbol{x}) \geq f(\boldsymbol{y} + a \boldsymbol{\chi}_i) - f(\boldsymbol{y}) \). By applying the submodularity equivalence from Lemma 2.2 in \cite{soma2014optimal}, submodularity follows from the diminishing returns property, thus establishing the "if" direction.

\textbf{``Only If'' Direction.}
Assume \(f\) is submodular, meaning \( f(\boldsymbol{a}) + f(\boldsymbol{b}) \geq f(\boldsymbol{a} \vee \boldsymbol{b}) + f(\boldsymbol{a} \wedge \boldsymbol{b}) \) for all \(\boldsymbol{a}, \boldsymbol{b} \in \mathbb{N}_{0}^U\). Let \(\boldsymbol{a} = \boldsymbol{x} + \boldsymbol{\chi}_i\) and \(\boldsymbol{b} = \boldsymbol{x} + \boldsymbol{\chi}_j\) for distinct \(i, j \in \cU\). Then, \(\boldsymbol{a} \vee \boldsymbol{b} = \boldsymbol{x} + \boldsymbol{\chi}_i + \boldsymbol{\chi}_j\) and \(\boldsymbol{a} \wedge \boldsymbol{b} = \boldsymbol{x}\). By submodularity, \( f(\boldsymbol{x} + \boldsymbol{\chi}_i) + f(\boldsymbol{x} + \boldsymbol{\chi}_j) \geq f(\boldsymbol{x} + \boldsymbol{\chi}_i + \boldsymbol{\chi}_j) + f(\boldsymbol{x}) \), which rearranges to \( f(\boldsymbol{x} + \boldsymbol{\chi}_i) - f(\boldsymbol{x}) \geq f(\boldsymbol{x} + \boldsymbol{\chi}_i + \boldsymbol{\chi}_j) - f(\boldsymbol{x} + \boldsymbol{\chi}_j) \), as required.
\end{proof}

\section{More Experiment Details}\label{app:experiment}

The curated datasets, categorized into three major groups—Diet, Apparel, and IP-themed brands—capture co-branding relationships among brand pairs. \Cref{fig:visu} visualizes these relationships by quantifying the number of co-branding occurrences, providing insights into the feasibility and frequency of such collaborations.


\begin{thebibliography}{39}


\ifx \showCODEN    \undefined \def \showCODEN     #1{\unskip}     \fi
\ifx \showDOI      \undefined \def \showDOI       #1{#1}\fi
\ifx \showISBNx    \undefined \def \showISBNx     #1{\unskip}     \fi
\ifx \showISBNxiii \undefined \def \showISBNxiii  #1{\unskip}     \fi
\ifx \showISSN     \undefined \def \showISSN      #1{\unskip}     \fi
\ifx \showLCCN     \undefined \def \showLCCN      #1{\unskip}     \fi
\ifx \shownote     \undefined \def \shownote      #1{#1}          \fi
\ifx \showarticletitle \undefined \def \showarticletitle #1{#1}   \fi
\ifx \showURL      \undefined \def \showURL       {\relax}        \fi
\providecommand\bibfield[2]{#2}
\providecommand\bibinfo[2]{#2}
\providecommand\natexlab[1]{#1}
\providecommand\showeprint[2][]{arXiv:#2}

\bibitem[Agrawal et~al\mbox{.}(2018)]%
        {agrawal2018proportional}
\bibfield{author}{\bibinfo{person}{Shipra Agrawal}, \bibinfo{person}{Morteza Zadimoghaddam}, {and} \bibinfo{person}{Vahab Mirrokni}.} \bibinfo{year}{2018}\natexlab{}.
\newblock \showarticletitle{Proportional allocation: Simple, distributed, and diverse matching with high entropy}. In \bibinfo{booktitle}{\emph{International Conference on Machine Learning}}. PMLR, \bibinfo{pages}{99--108}.
\newblock


\bibitem[Alon et~al\mbox{.}(2012)]%
        {alon2012optimizing}
\bibfield{author}{\bibinfo{person}{Noga Alon}, \bibinfo{person}{Iftah Gamzu}, {and} \bibinfo{person}{Moshe Tennenholtz}.} \bibinfo{year}{2012}\natexlab{}.
\newblock \showarticletitle{Optimizing budget allocation among channels and influencers}. In \bibinfo{booktitle}{\emph{Proceedings of the 21st international conference on World Wide Web}}. ACM, \bibinfo{pages}{381--388}.
\newblock


\bibitem[Bobadilla et~al\mbox{.}(2013)]%
        {bobadilla2013recommender}
\bibfield{author}{\bibinfo{person}{Jes{\'u}s Bobadilla}, \bibinfo{person}{Fernando Ortega}, \bibinfo{person}{Antonio Hernando}, {and} \bibinfo{person}{Abraham Guti{\'e}rrez}.} \bibinfo{year}{2013}\natexlab{}.
\newblock \showarticletitle{Recommender systems survey}.
\newblock \bibinfo{journal}{\emph{Knowledge-based systems}}  \bibinfo{volume}{46} (\bibinfo{year}{2013}), \bibinfo{pages}{109--132}.
\newblock


\bibitem[Chen et~al\mbox{.}(2013)]%
        {chen2013combinatorial}
\bibfield{author}{\bibinfo{person}{Wei Chen}, \bibinfo{person}{Yajun Wang}, {and} \bibinfo{person}{Yang Yuan}.} \bibinfo{year}{2013}\natexlab{}.
\newblock \showarticletitle{Combinatorial multi-armed bandit: General framework and applications}. In \bibinfo{booktitle}{\emph{International Conference on Machine Learning}}. PMLR, \bibinfo{pages}{151--159}.
\newblock


\bibitem[Chen et~al\mbox{.}(2016)]%
        {chen2016combinatorial}
\bibfield{author}{\bibinfo{person}{Wei Chen}, \bibinfo{person}{Yajun Wang}, \bibinfo{person}{Yang Yuan}, {and} \bibinfo{person}{Qinshi Wang}.} \bibinfo{year}{2016}\natexlab{}.
\newblock \showarticletitle{Combinatorial multi-armed bandit and its extension to probabilistically triggered arms}.
\newblock \bibinfo{journal}{\emph{The Journal of Machine Learning Research}} \bibinfo{volume}{17}, \bibinfo{number}{1} (\bibinfo{year}{2016}), \bibinfo{pages}{1746--1778}.
\newblock


\bibitem[Dai et~al\mbox{.}(2024a)]%
        {dai2024cost}
\bibfield{author}{\bibinfo{person}{Xiangxiang Dai}, \bibinfo{person}{Jin Li}, \bibinfo{person}{Xutong Liu}, \bibinfo{person}{Anqi Yu}, {and} \bibinfo{person}{John Lui}.} \bibinfo{year}{2024}\natexlab{a}.
\newblock \showarticletitle{Cost-Effective Online Multi-LLM Selection with Versatile Reward Models}.
\newblock \bibinfo{journal}{\emph{arXiv preprint arXiv:2405.16587}} (\bibinfo{year}{2024}).
\newblock


\bibitem[Dai et~al\mbox{.}(2025)]%
        {dai2025variance}
\bibfield{author}{\bibinfo{person}{Xiangxiang Dai}, \bibinfo{person}{Xutong Liu}, \bibinfo{person}{Jinhang Zuo}, \bibinfo{person}{Hong Xie}, \bibinfo{person}{Carlee Joe-Wong}, {and} \bibinfo{person}{John C.~S. Lui}.} \bibinfo{year}{2025}\natexlab{}.
\newblock \showarticletitle{Variance-Aware Bandit Framework for Dynamic Probabilistic Maximum Coverage Problem With Triggered or Self-Reliant Arms}.
\newblock \bibinfo{journal}{\emph{IEEE Transactions on Networking}} (\bibinfo{year}{2025}), \bibinfo{pages}{1--12}.
\newblock


\bibitem[Dai et~al\mbox{.}(2024b)]%
        {dai2024conversational}
\bibfield{author}{\bibinfo{person}{Xiangxiang Dai}, \bibinfo{person}{Zhiyong Wang}, \bibinfo{person}{Jize Xie}, \bibinfo{person}{Xutong Liu}, {and} \bibinfo{person}{John~CS Lui}.} \bibinfo{year}{2024}\natexlab{b}.
\newblock \showarticletitle{Conversational Recommendation with Online Learning and Clustering on Misspecified Users}.
\newblock \bibinfo{journal}{\emph{IEEE Transactions on Knowledge and Data Engineering}} \bibinfo{volume}{36}, \bibinfo{number}{12} (\bibinfo{year}{2024}), \bibinfo{pages}{7825--7838}.
\newblock


\bibitem[Dai et~al\mbox{.}(2024c)]%
        {dai2024online}
\bibfield{author}{\bibinfo{person}{Xiangxiang Dai}, \bibinfo{person}{Zhiyong Wang}, \bibinfo{person}{Jize Xie}, \bibinfo{person}{Tong Yu}, {and} \bibinfo{person}{John~CS Lui}.} \bibinfo{year}{2024}\natexlab{c}.
\newblock \showarticletitle{Online Learning and Detecting Corrupted Users for Conversational Recommendation Systems}.
\newblock \bibinfo{journal}{\emph{IEEE Transactions on Knowledge and Data Engineering}} \bibinfo{volume}{36}, \bibinfo{number}{12} (\bibinfo{year}{2024}), \bibinfo{pages}{8939--8953}.
\newblock


\bibitem[Faury et~al\mbox{.}(2020)]%
        {faury2020improved}
\bibfield{author}{\bibinfo{person}{Louis Faury}, \bibinfo{person}{Marc Abeille}, \bibinfo{person}{Cl{\'e}ment Calauz{\`e}nes}, {and} \bibinfo{person}{Olivier Fercoq}.} \bibinfo{year}{2020}\natexlab{}.
\newblock \showarticletitle{Improved optimistic algorithms for logistic bandits}. In \bibinfo{booktitle}{\emph{International Conference on Machine Learning}}. PMLR, \bibinfo{pages}{3052--3060}.
\newblock


\bibitem[Gao et~al\mbox{.}(2023)]%
        {Gao2023AlleviatingME}
\bibfield{author}{\bibinfo{person}{Chongming Gao}, \bibinfo{person}{Kexin Huang}, \bibinfo{person}{Jiawei Chen}, \bibinfo{person}{Yuan Zhang}, \bibinfo{person}{Biao Li}, \bibinfo{person}{Peng Jiang}, \bibinfo{person}{Shiqin Wang}, \bibinfo{person}{Zhong Zhang}, {and} \bibinfo{person}{Xiangnan He}.} \bibinfo{year}{2023}\natexlab{}.
\newblock \showarticletitle{Alleviating Matthew Effect of Offline Reinforcement Learning in Interactive Recommendation}.
\newblock \bibinfo{journal}{\emph{Proceedings of the 46th International ACM SIGIR Conference on Research and Development in Information Retrieval}} (\bibinfo{year}{2023}).
\newblock


\bibitem[Goda et~al\mbox{.}(2023)]%
        {goda2023dynamic}
\bibfield{author}{\bibinfo{person}{Dileep~Reddy Goda}, \bibinfo{person}{Vishal~Reddy Vadiyala}, \bibinfo{person}{Sridhar~Reddy Yerram}, {and} \bibinfo{person}{Suman~Reddy Mallipeddi}.} \bibinfo{year}{2023}\natexlab{}.
\newblock \showarticletitle{Dynamic Programming Approaches for Resource Allocation in Project Scheduling: Maximizing Efficiency under Time and Budget Constraints}.
\newblock \bibinfo{journal}{\emph{ABC Journal of Advanced Research}} \bibinfo{volume}{12}, \bibinfo{number}{1} (\bibinfo{year}{2023}), \bibinfo{pages}{1--16}.
\newblock


\bibitem[Gogri(2022)]%
        {gogri2022co}
\bibfield{author}{\bibinfo{person}{Sonal Gogri}.} \bibinfo{year}{2022}\natexlab{}.
\newblock \showarticletitle{Co-branding: A strategic decision in a competitive world}.
\newblock \bibinfo{journal}{\emph{EPRA International Journal of Economics, Business and Management Studies}}  \bibinfo{volume}{9} (\bibinfo{year}{2022}), \bibinfo{pages}{20--23}.
\newblock


\bibitem[Helmig et~al\mbox{.}(2007)]%
        {helmig2007explaining}
\bibfield{author}{\bibinfo{person}{Bernd Helmig}, \bibinfo{person}{Jan-Alexander Huber}, {and} \bibinfo{person}{Peter Leeflang}.} \bibinfo{year}{2007}\natexlab{}.
\newblock \showarticletitle{Explaining behavioural intentions toward co-branded products}.
\newblock \bibinfo{journal}{\emph{Journal of Marketing Management}} \bibinfo{volume}{23}, \bibinfo{number}{3-4} (\bibinfo{year}{2007}), \bibinfo{pages}{285--304}.
\newblock


\bibitem[Heyden et~al\mbox{.}(2024)]%
        {heyden2024budgeted}
\bibfield{author}{\bibinfo{person}{Marco Heyden}, \bibinfo{person}{Vadim Arzamasov}, \bibinfo{person}{Edouard Fouch{\'e}}, {and} \bibinfo{person}{Klemens B{\"o}hm}.} \bibinfo{year}{2024}\natexlab{}.
\newblock \showarticletitle{Budgeted Multi-Armed Bandits with Asymmetric Confidence Intervals}. In \bibinfo{booktitle}{\emph{Proceedings of the 30th ACM SIGKDD Conference on Knowledge Discovery and Data Mining}}. \bibinfo{pages}{1073--1084}.
\newblock


\bibitem[Hochba(1997)]%
        {hochba1997approximation}
\bibfield{author}{\bibinfo{person}{Dorit~S Hochba}.} \bibinfo{year}{1997}\natexlab{}.
\newblock \showarticletitle{Approximation algorithms for NP-hard problems}.
\newblock \bibinfo{journal}{\emph{ACM Sigact News}} \bibinfo{volume}{28}, \bibinfo{number}{2} (\bibinfo{year}{1997}), \bibinfo{pages}{40--52}.
\newblock


\bibitem[Johnson et~al\mbox{.}(2014)]%
        {johnson2014logistic}
\bibfield{author}{\bibinfo{person}{Christopher~C Johnson} {et~al\mbox{.}}} \bibinfo{year}{2014}\natexlab{}.
\newblock \showarticletitle{Logistic matrix factorization for implicit feedback data}.
\newblock \bibinfo{journal}{\emph{Advances in Neural Information Processing Systems}} \bibinfo{volume}{27}, \bibinfo{number}{78} (\bibinfo{year}{2014}), \bibinfo{pages}{1--9}.
\newblock


\bibitem[Khuller et~al\mbox{.}(1999)]%
        {khuller1999budgeted}
\bibfield{author}{\bibinfo{person}{Samir Khuller}, \bibinfo{person}{Anna Moss}, {and} \bibinfo{person}{Joseph~Seffi Naor}.} \bibinfo{year}{1999}\natexlab{}.
\newblock \showarticletitle{The budgeted maximum coverage problem}.
\newblock \bibinfo{journal}{\emph{Information processing letters}} \bibinfo{volume}{70}, \bibinfo{number}{1} (\bibinfo{year}{1999}), \bibinfo{pages}{39--45}.
\newblock


\bibitem[Lattimore and Szepesv{\'a}ri(2020)]%
        {lattimore2020bandit}
\bibfield{author}{\bibinfo{person}{Tor Lattimore} {and} \bibinfo{person}{Csaba Szepesv{\'a}ri}.} \bibinfo{year}{2020}\natexlab{}.
\newblock \bibinfo{booktitle}{\emph{Bandit algorithms}}.
\newblock \bibinfo{publisher}{Cambridge University Press}.
\newblock


\bibitem[Li et~al\mbox{.}(2016)]%
        {li2016contextual}
\bibfield{author}{\bibinfo{person}{Shuai Li}, \bibinfo{person}{Baoxiang Wang}, \bibinfo{person}{Shengyu Zhang}, {and} \bibinfo{person}{Wei Chen}.} \bibinfo{year}{2016}\natexlab{}.
\newblock \showarticletitle{Contextual combinatorial cascading bandits}. In \bibinfo{booktitle}{\emph{International conference on machine learning}}. PMLR, \bibinfo{pages}{1245--1253}.
\newblock


\bibitem[Li et~al\mbox{.}(2025a)]%
        {li2025demystifying}
\bibfield{author}{\bibinfo{person}{Zhuohua Li}, \bibinfo{person}{Maoli Liu}, \bibinfo{person}{Xiangxiang Dai}, {and} \bibinfo{person}{John~C.S. Lui}.} \bibinfo{year}{2025}\natexlab{a}.
\newblock \showarticletitle{Demystifying Online Clustering of Bandits: Enhanced Exploration Under Stochastic and Smoothed Adversarial Contexts}. In \bibinfo{booktitle}{\emph{Proceedings of the International Conference on Learning Representations}}.
\newblock


\bibitem[Li et~al\mbox{.}(2025b)]%
        {li2025towards}
\bibfield{author}{\bibinfo{person}{Zhuohua Li}, \bibinfo{person}{Maoli Liu}, \bibinfo{person}{Xiangxiang Dai}, {and} \bibinfo{person}{John~CS Lui}.} \bibinfo{year}{2025}\natexlab{b}.
\newblock \showarticletitle{Towards Efficient Conversational Recommendations: Expected Value of Information Meets Bandit Learning}. In \bibinfo{booktitle}{\emph{Proceedings of the ACM on Web Conference 2025}}. \bibinfo{pages}{4226--4238}.
\newblock


\bibitem[Lite(2024)]%
        {SocialBeta}
\bibfield{author}{\bibinfo{person}{SocialBeta Lite}.} \bibinfo{year}{2024}\natexlab{}.
\newblock \bibinfo{howpublished}{\url{https://socialbeta.com}}.
\newblock


\bibitem[Liu et~al\mbox{.}(2024)]%
        {liu2024combinatorial}
\bibfield{author}{\bibinfo{person}{Xutong Liu}, \bibinfo{person}{Xiangxiang Dai}, \bibinfo{person}{Xuchuang Wang}, \bibinfo{person}{Mohammad Hajiesmaili}, {and} \bibinfo{person}{John Lui}.} \bibinfo{year}{2024}\natexlab{}.
\newblock \showarticletitle{Combinatorial Logistic Bandits}.
\newblock \bibinfo{journal}{\emph{arXiv preprint arXiv:2410.17075}} (\bibinfo{year}{2024}).
\newblock


\bibitem[Liu et~al\mbox{.}(2025)]%
        {liu2025offline}
\bibfield{author}{\bibinfo{person}{Xutong Liu}, \bibinfo{person}{Xiangxiang Dai}, \bibinfo{person}{Jinhang Zuo}, \bibinfo{person}{Siwei Wang}, \bibinfo{person}{Carlee-Joe Wong}, \bibinfo{person}{John Lui}, {and} \bibinfo{person}{Wei Chen}.} \bibinfo{year}{2025}\natexlab{}.
\newblock \showarticletitle{Offline Learning for Combinatorial Multi-armed Bandits}.
\newblock \bibinfo{journal}{\emph{arXiv preprint arXiv:2501.19300}} (\bibinfo{year}{2025}).
\newblock


\bibitem[Liu et~al\mbox{.}(2021)]%
        {liu2021multi}
\bibfield{author}{\bibinfo{person}{Xutong Liu}, \bibinfo{person}{Jinhang Zuo}, \bibinfo{person}{Xiaowei Chen}, \bibinfo{person}{Wei Chen}, {and} \bibinfo{person}{John~CS Lui}.} \bibinfo{year}{2021}\natexlab{}.
\newblock \showarticletitle{Multi-layered Network Exploration via Random Walks: From Offline Optimization to Online Learning}. In \bibinfo{booktitle}{\emph{International Conference on Machine Learning}}. PMLR, \bibinfo{pages}{7057--7066}.
\newblock


\bibitem[Liu et~al\mbox{.}(2022)]%
        {liu2022batch}
\bibfield{author}{\bibinfo{person}{Xutong Liu}, \bibinfo{person}{Jinhang Zuo}, \bibinfo{person}{Siwei Wang}, \bibinfo{person}{Carlee Joe-Wong}, \bibinfo{person}{John Lui}, {and} \bibinfo{person}{Wei Chen}.} \bibinfo{year}{2022}\natexlab{}.
\newblock \showarticletitle{Batch-size independent regret bounds for combinatorial semi-bandits with probabilistically triggered arms or independent arms}.
\newblock \bibinfo{journal}{\emph{Advances in Neural Information Processing Systems}}  \bibinfo{volume}{35} (\bibinfo{year}{2022}), \bibinfo{pages}{14904--14916}.
\newblock


\bibitem[Low and Mohr(1998)]%
        {low1998brand}
\bibfield{author}{\bibinfo{person}{George~S Low} {and} \bibinfo{person}{Jakki~J Mohr}.} \bibinfo{year}{1998}\natexlab{}.
\newblock \bibinfo{booktitle}{\emph{Brand managers' perceptions of the marketing communications budget allocation process}}.
\newblock \bibinfo{publisher}{Marketing Science Institute}.
\newblock


\bibitem[Nemhauser and Ullmann(1969)]%
        {nemhauser1969discrete}
\bibfield{author}{\bibinfo{person}{George~L Nemhauser} {and} \bibinfo{person}{Zev Ullmann}.} \bibinfo{year}{1969}\natexlab{}.
\newblock \showarticletitle{Discrete dynamic programming and capital allocation}.
\newblock \bibinfo{journal}{\emph{Management Science}} \bibinfo{volume}{15}, \bibinfo{number}{9} (\bibinfo{year}{1969}), \bibinfo{pages}{494--505}.
\newblock


\bibitem[Pun and Heese(2015)]%
        {pun2015note}
\bibfield{author}{\bibinfo{person}{Hubert Pun} {and} \bibinfo{person}{H~Sebastian Heese}.} \bibinfo{year}{2015}\natexlab{}.
\newblock \showarticletitle{A note on budget allocation for market research and advertising}.
\newblock \bibinfo{journal}{\emph{International Journal of Production Economics}}  \bibinfo{volume}{166} (\bibinfo{year}{2015}), \bibinfo{pages}{85--89}.
\newblock


\bibitem[Singh(2021)]%
        {singh2021reinforcement}
\bibfield{author}{\bibinfo{person}{Anupam Singh}.} \bibinfo{year}{2021}\natexlab{}.
\newblock \showarticletitle{Reinforcement learning based empirical comparison of ucb, epsilon-greedy, and thompson sampling}.
\newblock \bibinfo{journal}{\emph{Int. J. of Aquatic Science}} \bibinfo{volume}{12}, \bibinfo{number}{2} (\bibinfo{year}{2021}), \bibinfo{pages}{2961--2969}.
\newblock


\bibitem[Soma et~al\mbox{.}(2014)]%
        {soma2014optimal}
\bibfield{author}{\bibinfo{person}{Tasuku Soma}, \bibinfo{person}{Naonori Kakimura}, \bibinfo{person}{Kazuhiro Inaba}, {and} \bibinfo{person}{Ken-ichi Kawarabayashi}.} \bibinfo{year}{2014}\natexlab{}.
\newblock \showarticletitle{Optimal budget allocation: Theoretical guarantee and efficient algorithm}. In \bibinfo{booktitle}{\emph{International Conference on Machine Learning}}. \bibinfo{pages}{351--359}.
\newblock


\bibitem[Turan(2022)]%
        {turan2022deal}
\bibfield{author}{\bibinfo{person}{Ceyda~Paydas Turan}.} \bibinfo{year}{2022}\natexlab{}.
\newblock \showarticletitle{Deal or deny: The effectiveness of crisis response strategies on brand equity of the focal brand in co-branding}.
\newblock \bibinfo{journal}{\emph{Journal of Business Research}}  \bibinfo{volume}{149} (\bibinfo{year}{2022}), \bibinfo{pages}{615--629}.
\newblock


\bibitem[Wan et~al\mbox{.}(2023)]%
        {wan2023bandit}
\bibfield{author}{\bibinfo{person}{Zongqi Wan}, \bibinfo{person}{Jialin Zhang}, \bibinfo{person}{Wei Chen}, \bibinfo{person}{Xiaoming Sun}, {and} \bibinfo{person}{Zhijie Zhang}.} \bibinfo{year}{2023}\natexlab{}.
\newblock \showarticletitle{Bandit multi-linear DR-submodular maximization and its applications on adversarial submodular bandits}. In \bibinfo{booktitle}{\emph{International Conference on Machine Learning}}. PMLR, \bibinfo{pages}{35491--35524}.
\newblock


\bibitem[Wang and Chen(2017)]%
        {wang2017improving}
\bibfield{author}{\bibinfo{person}{Qinshi Wang} {and} \bibinfo{person}{Wei Chen}.} \bibinfo{year}{2017}\natexlab{}.
\newblock \showarticletitle{Improving regret bounds for combinatorial semi-bandits with probabilistically triggered arms and its applications}. In \bibinfo{booktitle}{\emph{Advances in Neural Information Processing Systems}}. \bibinfo{pages}{1161--1171}.
\newblock


\bibitem[Wang and Chen(2018)]%
        {wang2018thompson}
\bibfield{author}{\bibinfo{person}{Siwei Wang} {and} \bibinfo{person}{Wei Chen}.} \bibinfo{year}{2018}\natexlab{}.
\newblock \showarticletitle{Thompson Sampling for Combinatorial Semi-Bandits}. In \bibinfo{booktitle}{\emph{International Conference on Machine Learning}}. \bibinfo{pages}{5114--5122}.
\newblock


\bibitem[Xia et~al\mbox{.}(2023)]%
        {xia2024kdd}
\bibfield{author}{\bibinfo{person}{Yu Xia}, \bibinfo{person}{Junda Wu}, \bibinfo{person}{Tong Yu}, \bibinfo{person}{Sungchul Kim}, \bibinfo{person}{Ryan~A. Rossi}, {and} \bibinfo{person}{Shuai Li}.} \bibinfo{year}{2023}\natexlab{}.
\newblock \showarticletitle{User-Regulation Deconfounded Conversational Recommender System with Bandit Feedback}. In \bibinfo{booktitle}{\emph{Proceedings of the 29th ACM SIGKDD Conference on Knowledge Discovery and Data Mining}}. \bibinfo{pages}{2694–2704}.
\newblock


\bibitem[Xu et~al\mbox{.}(2021)]%
        {xu2021dual}
\bibfield{author}{\bibinfo{person}{Lily Xu}, \bibinfo{person}{Elizabeth Bondi}, \bibinfo{person}{Fei Fang}, \bibinfo{person}{Andrew Perrault}, \bibinfo{person}{Kai Wang}, {and} \bibinfo{person}{Milind Tambe}.} \bibinfo{year}{2021}\natexlab{}.
\newblock \showarticletitle{Dual-mandate patrols: Multi-armed bandits for green security}. In \bibinfo{booktitle}{\emph{Proceedings of the AAAI Conference on Artificial Intelligence}}, Vol.~\bibinfo{volume}{35}. \bibinfo{pages}{14974--14982}.
\newblock


\bibitem[Yang et~al\mbox{.}(2024)]%
        {yang2024conversational}
\bibfield{author}{\bibinfo{person}{Shuhua Yang}, \bibinfo{person}{Hui Yuan}, \bibinfo{person}{Xiaoying Zhang}, \bibinfo{person}{Mengdi Wang}, \bibinfo{person}{Hong Zhang}, {and} \bibinfo{person}{Huazheng Wang}.} \bibinfo{year}{2024}\natexlab{}.
\newblock \showarticletitle{Conversational Dueling Bandits in Generalized Linear Models}. In \bibinfo{booktitle}{\emph{Proceedings of the 30th ACM SIGKDD Conference on Knowledge Discovery and Data Mining}}. \bibinfo{pages}{3806--3817}.
\newblock


\end{thebibliography}
\end{document}